\documentclass[twoside]{article}

\usepackage[accepted]{aistats2023}

\usepackage[utf8]{inputenc} 
\usepackage[T1]{fontenc}    
\usepackage{hyperref}       
\usepackage{url}            
\usepackage{booktabs}       
\usepackage{amsfonts}       
\usepackage{nicefrac}       
\usepackage{microtype}      
\usepackage{xcolor}         
\usepackage{maths}
\usepackage{natbib}
\usepackage{wrapfig}
\usepackage{longtable}
\usepackage{xcolor}

\hypersetup{
	colorlinks,
	linkcolor={red!40!gray},
	citecolor={blue!40!gray},
	urlcolor={blue!70!gray}
}

\allowdisplaybreaks
\begin{document}
	
	%
	
	%
	
	\twocolumn[
	
	\aistatstitle{Active Exploration via Experiment Design in  Markov Chains}
	
	\aistatsauthor{ %
		Mojm\'ir Mutn\'y\\
		\texttt{mojmir.mutny@inf.ethz.ch} 
		\And
		Tadeusz Janik\\
		\texttt{ tjanik@student.ethz.ch} 
		\And
		Andreas Krause \\
		\texttt{krausea@ethz.ch}  }
	
	\aistatsaddress{ 		ETH Z\"urich \And  		ETH Z\"urich \And 		ETH Z\"urich  } ]
	
\begin{abstract}
	\looseness -1 A key challenge in science and engineering is to design experiments to learn about some unknown quantity of interest. Classical experimental design optimally allocates the experimental budget to maximize a notion of utility (e.g., reduction in uncertainty about the unknown quantity).  We consider a rich setting, where the experiments are associated with states in a {\em Markov chain}, and we can only choose them by selecting a {\em policy} controlling the state transitions. This problem captures important applications, from exploration in reinforcement learning to spatial monitoring tasks. We propose an algorithm -- \textsc{markov-design} -- that efficiently selects policies whose measurement allocation \emph{provably converges to the optimal one}. The algorithm is sequential in nature, adapting its choice of policies (experiments) informed by past measurements. In addition to our theoretical analysis, we showcase our framework on applications in ecological surveillance and pharmacology.
\end{abstract}
	\vspace{-0.2cm}
\section{Introduction}
\looseness -1 The optimal design of experiments \citep{Pukelsheim2006, Chaloner1995} is a ubiquitous challenge in science and engineering. The key goal is to utilize the limited budget (time, resources) to gain as much information about some unknown quantity of interest.
\looseness -1 Classical experiment design assumes that an experiment measures a single value with specific conditions. 
Motivated by applications illustrated in more detail below, we assume that experiments are associated with \emph{policies} executed in a \emph{known Markov chain}. There are three major challenges in finding the set of best policies. Firstly, the space of policies can be {\em combinatorial} in the size of the state-action space, and searching over it directly would lead to intractable optimization problems with classical methods. Secondly, the feedback from a policy is {\em stochastic}, due to the randomness of the Markov chain, which needs to be taken into account. Lastly, it is unclear how classical experimental design {\em objectives} can be formulated as a function of the policies. We address these challenges with convex optimization techniques. 
\vspace{-0.1cm}
\paragraph{Quantity of interest} \looseness -1 The general goal of experiment design is to estimate some aspects of an unknown quantity $f$ of interest. Here, we assume $f$ is a function of states and actions of a Markov chain $f: \mX \times \mA \rightarrow \mR$. We further assume $f$ belongs to a reproducing kernel Hilbert space (RKHS) ($f\in \mH_k$) with a \emph{known kernel} $k((x,a),(x',a')) = \Phi(x,a)^\top \Phi(x',a')$ \footnote{$\cdot^\top$ denotes the Hilbert space inner product.}, where $x,x' \in \mX$ and $a,a'\in \mA$, and a \emph{known} bound $\norm{f}_{\mH_k} \leq \frac{1}{\lambda}$. As a concrete example, the function $f$ can model the distribution of species over a certain geographical location, where the kernel incorporates spatial features such as access to water, soil salinity, etc., and $x$ corresponds to the spatial location. While $f$ is unknown altogether, we are often interested in estimating only a \emph{linear functional} of it $\bC f$, where $\bC:\mH_k \rightarrow \mR^p$. In the context of biological surveillance, this can be, e.g., a spatial average over certain locations. Another example are values at specific locations, where the rows of $\bC_{i:} = \Phi(x_i,\cdot)^\top$, with $\{x_i\}_i$ being the locations of interest. For simpler models (finite-dimensional $f$), $\bC$ may even be the identity (i.e., the goal is to estimate $f$ completely). We observe $f$ via noisy evaluations at specific states $x$ while performing an action $a$, 
\begin{equation}\label{eq:model}\looseness -1
	y = f^\top \Phi(x, a) + \epsilon
\end{equation}
where $\epsilon$ is random noise realization such that $\mE[\epsilon]=0$.
\vspace{-0.1cm}
\paragraph{Exploration} \looseness -1 Consider an example, where we want to deploy an agent (e.g., a drone) that explores the environment efficiently to learn $\bC f$ from observations at $(x,a)$. Due to the kernel regularity assumption (RKHS), we know that similar $(x, a)$ lead to similar values of $f$, and to understand $\bC f$, intuitively, we should explore diverse landscape features instead of evaluating similar $(x,a)$. In the species distribution example, these are illustrated via the pictograms in Fig.~\ref{fig:banner-density} and Fig.~\ref{fig:banner-traj}. We would like to pick a small subset of states covering all pictograms there, however, we cannot choose states arbitrarily: the only way we can choose them is by following the Markov transition rule, which then generates trajectories as in Fig.~\ref{fig:banner-traj}. In fact, to learn $\bC f$ effectively, we need to visit the states in proportion to the heat map in Fig.~\ref{fig:banner-density} - \emph{optimal visitation of the states}. In this work, we develop a method that picks policies sequentially such that their trajectories (as in Fig.~\ref{fig:banner-traj}) converge to the optimal visitation distribution over states and actions as in Fig.~\ref{fig:banner-density}, detailed below.
\begin{figure*}
	\centering
	\begin{subfigure}[t]{0.3\textwidth}
		\centering
		\includegraphics[width=\textwidth]{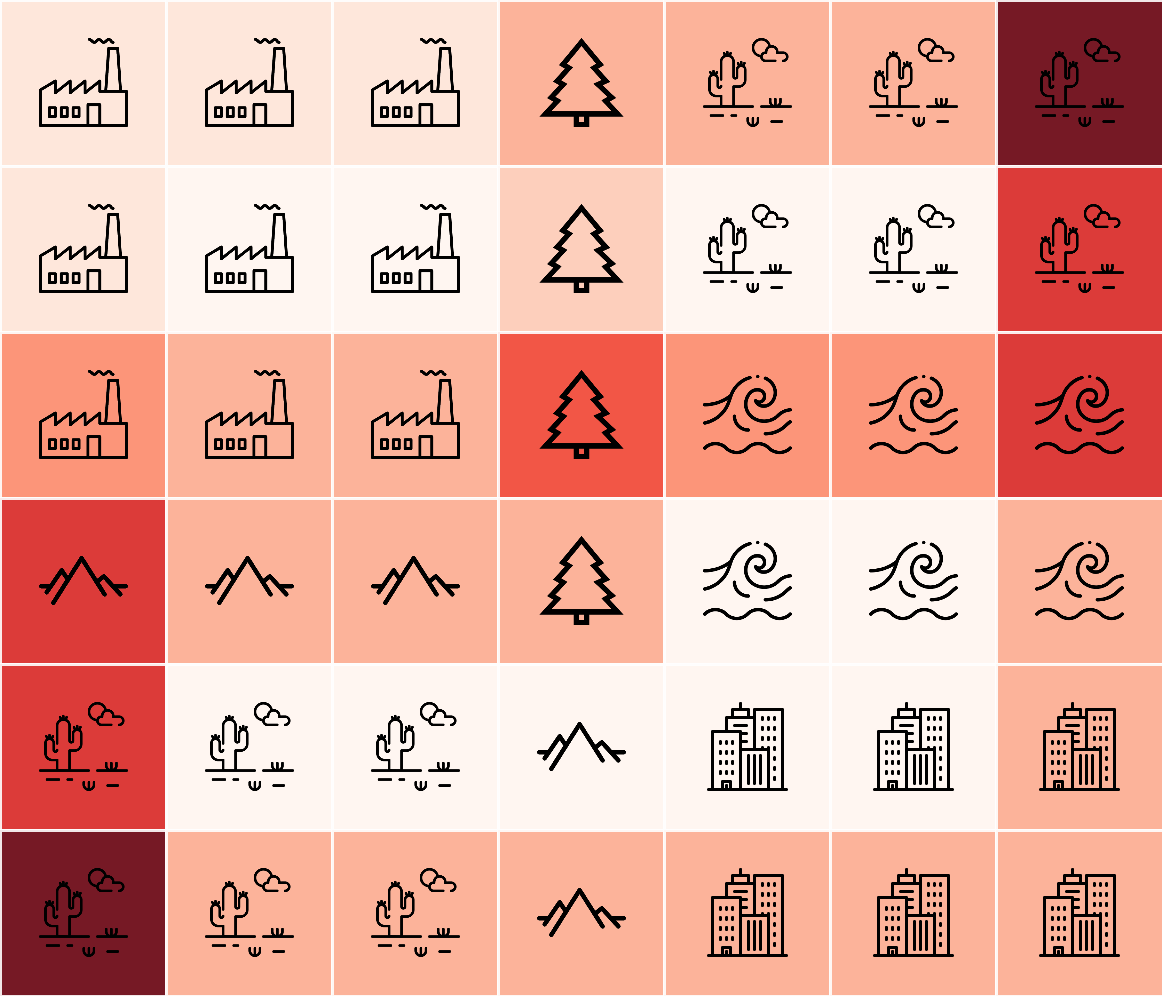}
		\caption{}
		\label{fig:banner-density}
	\end{subfigure}
	\begin{subfigure}[t]{0.3\textwidth}
		\centering
		\includegraphics[width=\textwidth]{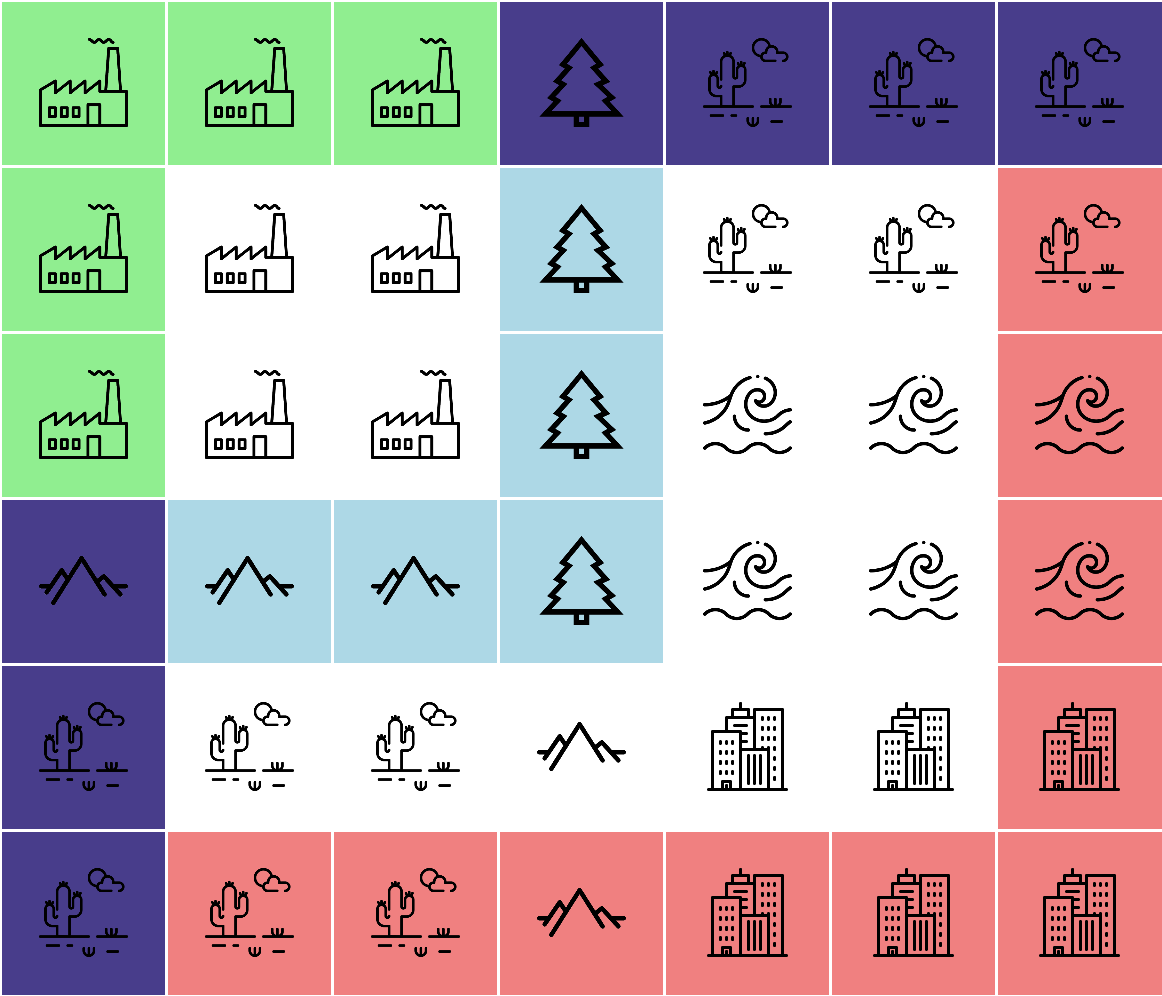}
		\caption{}
		\label{fig:banner-traj}
	\end{subfigure}
	\begin{subfigure}[t]{0.3\textwidth}
		\centering
		\includegraphics[width=\textwidth]{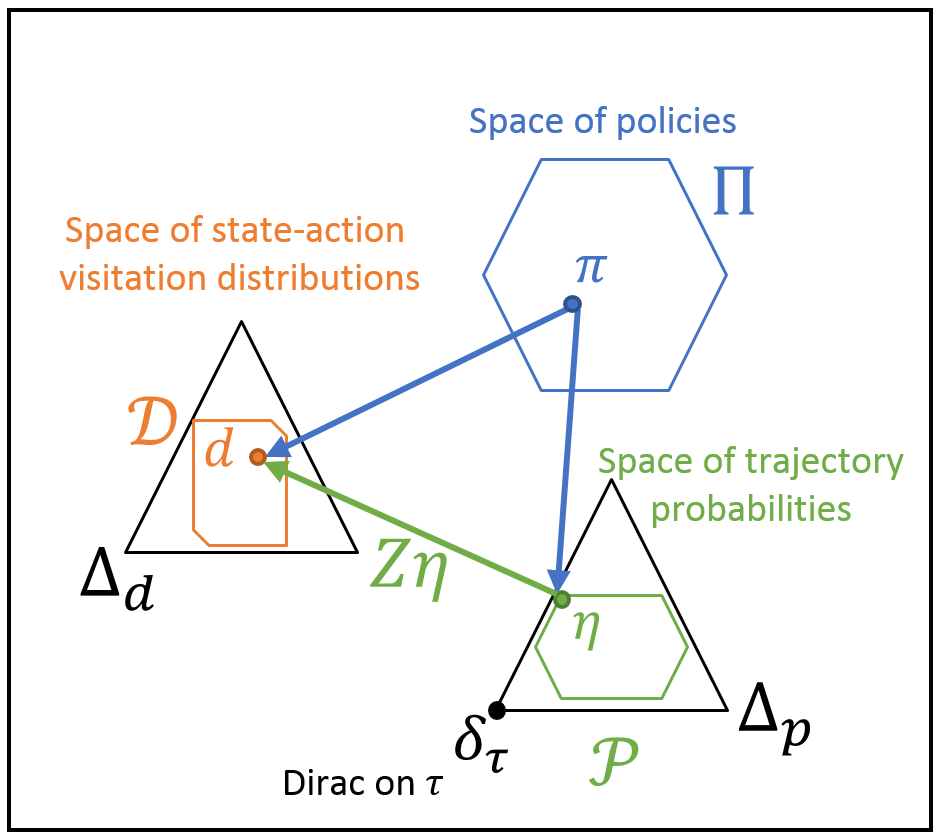}
		\caption{}
		\label{fig:banner-pic}
	\end{subfigure}
	\caption{a) Optimal state-action visitation distribution that maximizes the information about species distribution as a function of state features (indicated via pictograms). Different pictograms represent sectors, and hence different values of the unknown $f$. Note that some sectors are not covered frequently, as the information about the function value in these sectors is already extracted from different sectors with the same pictogram. b) Displayed are \emph{three trajectories} from \emph{three different policies} in \emph{three episodes}, seeking to match the optimal state-action distribution depicted in b). Note that the trajectories move from the initial (lower left) to the terminal state (upper right) and overlap in the purple-colored states. c) Relation between the sets of objects we work with. Given a policy $\pi$, we can calculate state-action visitations and probabilities over trajectories it generates. Not every state-action distribution nor probability over trajectories can be realized by a policy. Also, state-action visitations $d$ and probability over trajectories $\eta$ are related via a linear \emph{conversion map} $\bZ$. A single executed trajectory defines a Dirac delta on the probability space $\mP$, $\delta_\tau$.}
	\label{fig:banner}
	\vspace{-0.4cm}
\end{figure*}
\vspace{-0.2cm}
\paragraph{Contribution}
We design a \emph{novel} algorithm: \textsc{markov-design} for adaptive experiment design, where we collect data by deploying policies in a \emph{known} Markov chain for a fixed horizon. We notice that classical \emph{experimental design} objectives can be reformulated such that they depend on state-action visitations and show how they can be efficiently optimized using \emph{convex reward reinforcement learning} which reduces to a sequence of classical planning problems. The resulting algorithm relies on convex optimization techniques, and is simple to implement. We prove that our algorithm converges to the optimal allocation over trajectories generated by selected policies and propose several variants of the algorithm with varying levels of approximation and analyze their performance. Additionally, we demonstrate our algorithm on two real-world experimental design problems with Markov chain structure. 
\vspace{-0.2cm}
\section{Background and Related Work}
\vspace{-0.2cm}
\paragraph{Environment} \looseness -1 We assume a known set of environment states $\mX$ and actions $\mA$, and a known Markov transition model $p(x'|x,a): \mX \times \mA \rightarrow \mX$. We interact with the environment by executing a \emph{policy}, where actions are sampled $a_{h} \sim \pi_h(a|x_h)$, for horizon $H$, resulting in trajectories $\tau = \{x_0, a_0, x_1, a_1, \dots x_{H-1}, a_{H-1}\}$, where $\tau \in \mT$ is the set of all possible trajectories where the horizon is $H$. The space of \emph{all policies} is denoted $\Pi$. A policy induces a probability distribution over trajectories $\eta$, belonging to the feasible distribution set $\mP \subseteq \Delta_p$, a subset of all possible distributions over trajectories. A distribution supported only on one trajectory $\tau$ is denoted $\delta_\tau$. As in Fig.~\ref{fig:banner-pic}, note that a policy $\pi$ is related to a specific $\eta \in \mP$, but it is generally impossible to associate an arbitrary distribution over trajectories to a policy.

\vspace{-0.1cm}
\subsection{Related Work}\looseness-1 
\vspace{-0.1cm}
\paragraph{Experiment Design} \looseness-1 Optimal experimental design (OED) 
\citep{Pukelsheim2006} has a rich history, with modern applications in protein design \citep{Romero2013}, Poisson sensing \citep{Mutny2022a}, and many others \citep[cf.,][]{Chaloner1995}. It is closely related to active learning \citep{Settles2009} and bandit optimization \citep{Szepesvari2020}. Traditionally, in OED, the action (experiment) can be chosen deterministically, and the space of actions is not large. Scaling to many actions has been addressed in the bandit literature in the context of \emph{combinatorial bandits} \citep{Zou2014,Talebi2013,Jourdan2021}, where the action structures are, e.g., paths in a graph, or spanning trees. While the goal is optimization (i.e., finding $\argmax_x f(x)$), the techniques are similar in spirit. They circumvent the large action set problem by having an efficient oracle for incrementally increasing the allocation, e.g., the Dijkstra algorithm. Similarly, we provide a way to incrementally increase the visitations of state-actions, and our oracle relies on efficient Bellman optimality solvers.
\vspace{-0.2cm}
\paragraph{Convex RL} \looseness-1 Our algorithms utilize the recent framework of reinforcement learning with convex reward functions due to \citet{Hazan2019} and \citet{Zahavy2021}. These works propose a way to find a stationary policy $\pi^*$ which minimizes a certain convex reward function, which, in light of our work, can be uncertainty about $f$. Exploration with convex rewards in the context of MDPs is investigated by \citet{Tarbouriech2019} in non-episodic context with ergodicity assumption without the notion of similarity as studied here. Different from their work, we consider more complex objectives which incorporate the structure of RKHS, exploiting a connection with the field of experiment design. Shortly after the publication of the first version of this work, \citet{Wagenmaker2022} introduced a similar experiment design problem on Markov chains with the aim of developing instance-optimal RL algorithms. Unlike other works, we focus also on how the optimal policy $\pi^*$ should be executed. Replicating the policy $\pi^*$ for multiple episodes leads to potentially the same states being visited and a suboptimal convergence to the optimal choice of state-actions as we show. Instead, we incrementally construct a sequence of policies that take the history of executed trajectories into account and such that convergence of the empirical visitation distribution to the optimal visitation distribution of policy $\pi^*$ is faster than resampling. 

\vspace{-0.1cm} \looseness -1 If $f$ is interpreted as a reward function, our work can be seen as a reward learning problem similar to \citet{Lindner2021}. The difference here is that our exploration is planned for an entire episode instead of assuming an access to arbitrary state-action oracle (simulator) as in their work, which makes our approach more widely applicable. Episodic planning for reward learning with different assumptions on $f$ and methods is considered by \citet{Belogolovsky2021}. 
\vspace{-0.1cm}
\subsection{Estimation}\looseness -1 \vspace{-0.1cm}
We estimate $\bC f$ using a regularized least squares estimator, where the estimated $\hat{f}_t$ after $t$ episodes is the solution to
$ \hat{f}_t = \argmin_{f \in \mH_k}	\sum_{i=1}^t \sum_{(x,a) \in \tau_i} (f^\top \Phi(x,a) - y_{a,x,i})^2 + \lambda \norm{f}_{\mH_k}^2.$
In order to estimate the functional $\bC f$, we simply use $\bC \hat{f}$. This estimator is motivated by the famed Gauss-Markov theorem, as it minimizes the second moment of residuals \citep[see][for details]{Mutny2022}. Due to the representer theorem \citep[cf.,][]{Schoelkopf2001}, the above problem can be solved even if the RKHS is infinite-dimensional.

\vspace{-0.1cm}  \looseness -1 We aim to understand the uncertainty of the estimate $\bC\hat{f}$ as a function of trajectories taken. To do so, as common in optimal experiment design, we consider the second moment of the residuals  $\bC(\hat{f} -f)$ denoted as $\bE_T = \mE_\epsilon[\bC(f-\hat{f}) (f-\hat{f})^\top\bC^\top]$, where the expectation is understood over the noise realization $\epsilon$ (see Eq. \eqref{eq:model}). Evaluating it,  we can see its dependence on the executed trajectories $\tau_i, i\in[T]$:
\begin{equation}\label{eq:covariance}
	\bE_T \preceq \bC \left( \left(\sum_{i=1}^{T}I(\tau_i) + \lambda \bI \right)^{-1} \right)\bC^\top  
\end{equation}
where we call $I(\tau):\mH_k \rightarrow \mH_k\times \mH_k$ the {\em information matrix} for a trajectory $\tau$,
\begin{equation}\label{eq:info-matrix}
I(\tau) = \sum_{(a,x)\in \tau} \frac{1}{\sigma_{a,x}^2} \Phi(x,a)\Phi(x,a)^\top.
\end{equation}
and $\sigma_{a,x}$ denotes the variance of random variable $\epsilon$. The derivation follows a straightforward application of optimality conditions,  representer theorem, and use of bounded norm assumption $\norm{f}_{\mH_k} \leq \frac{1}{\lambda}$. A very detailed derivation is given by \citet{Mutny2022}. 

By appropriate choice of trajectories $\tau_i$, we can minimize the second moment of residuals in Eq.~\eqref{eq:covariance}. As alluded to in the beginning, unless the system is completely deterministic, we cannot directly choose trajectories $\tau_i$. Even if we could, this space of trajectories may grow {\em exponentially} in $|S|$ and $|A|$. Instead, we first reformulate the objective via a fractional allocation over trajectories $\eta \in \Delta_p$, essentially by inserting $1/T$ into Eq. \eqref{eq:covariance},
\begin{eqnarray}\label{eq:cov}
	\bE_T(\eta_T) & \hspace{-0.3cm}\preceq &  \hspace{-0.2cm} \frac{1}{T}\underbrace{\bC \left( \left( \sum_{\tau \in \mT} I(\tau)\eta_T(\tau) + \frac{\lambda}{T} \bI \right)^{-1} \right)\bC^\top}_{\text{covariance matrix } \bSigma(\eta_T)}
\end{eqnarray}
\looseness -1 where we identified $\eta_T = \frac{1}{T}\sum_{i=1}^T \delta_{\tau_i}$. Now the goal of experiment design is to invest the $T$ samples (trajectories) in such a way that the matrix in Eq.~\eqref{eq:cov} is as small as possible. 
Since the RHS is matrix-valued, 
a classical approach in OED is to scalarize the objective by one of the well-known scalarization functions (see \ref{table:targets}). A scalarization function $s:\mathbb{S}^{p\times p} \rightarrow \mR$ acts on the space of PSD matrices, and is convex.  These objectives usually correspond to a specific type of uncertainty, such as the squared error, entropy, or predictive error. Upon scalarization we have $s(\bE_T) \leq \frac{1}{T} s(\bSigma(\eta))$, where $s(\bSigma(\eta))$ represents the constant with which the budget $T$ is invested. We seek to have the lowest possible constant with $\eta^*$, $s(\bSigma(\eta^*))$.


\begin{table}
	\caption{Selected design objectives with their interpretations and name \citep{Fedorov1997}.}
	\label{table:targets}
	\scriptsize
	\centering
	\begin{tabular}{llll}
		\toprule
				Design     &  Represents  & $s(\bSigma)$ \\
				\midrule
				D  &  information  & $-\log\det\left(\bSigma^{-1} \right)$ \\
				A  &  parameter error$^2$  & $\Tr\left[ \bSigma \right]$ \\
				E  &  worst projection error$^2$  & $\lambda_{\max}\left[\bSigma \right]$ \\
		\bottomrule
	\end{tabular}
	\vspace{-0.3cm}
\end{table}

\vspace{-0.2cm}
\section{Experiment Design: Problem Statement}\label{sec:ee} \looseness -1
Our goal is to design and execute a sequence of $T$ policies $\pi_i$, $i\in [T]$, each one for a single episode, for a fixed number of steps $H$ (episode length), such that we reduce the uncertainty of $\bC \hat{f}$ efficiently. Directly optimizing probability distributions over trajectories in $F(\eta)$ is intractable as their description size is $(|\mX||\mA|)^H$ in the worst case. Instead, we focus on the building blocks of trajectories: states and actions, and their visitations. We design policies that visit the states and actions such that uncertainty is maximally reduced as measured by the function $F$ and show that for experiment design objectives the complexity of the problem can be reduced to optimization over distributions with size $|\mX||\mA|H$. 
\vspace{-0.2cm}
\subsection{State-action polytope} \looseness -1 To optimize the choice of policies (and hence trajectories), we work with an object that accumulates the information about state-visitations: the \emph{state-action visitation distribution}. For ease of exposition, assume that $\mX$ and $\mA$ are discrete sets. The space of all possible state-action visitation distributions with a known initial distribution $d_0$ is the \emph{state-action visitation polytope}, 
\begin{align*}
	\mD_h := \Big\{ d_h ~ | ~ d(x,a)\geq 0, ~ \sum_{a,x} d_h(a,x) = 1, \\ ~  \sum_a d_{h}(x',a) = \sum_{x,a}d_{h-1}(x,a)p(x'|x,a)   \Big\}.
\end{align*}
\looseness -1 For reference see \citet{Puterman2014} or 
\citet{Neu2020}. The \emph{average state-action visitation polytope} over a time horizon $H$ is the central constraint set used in this work, which we denote by $\mD := \{ d = \frac{1}{H}\sum_{h=1}^{H}d_h ~|~ d_i \in \mD_i ~ \forall i \in [H]\}$. Notice that for fixed horizon, the state-action visitation distribution $d_h$ for each timestep $h$ is different. A useful approximation is to let $H$ be very large. In that case, a reasonable approximation to this polytope is the \emph{average-case polytope}, where $\bar{\mD} = \{ d ~ | ~ d(x,a) \geq 0~ , \sum_{a,x} d(a,x) = 1, ~  \sum_a d(x',a) = \sum_{x,a}d(x,a)p(x'|x,a)   \}$, and the distribution $d$ for each $h$ is the same. The distributions $d_h$ can be generated via transition operator $P_{\pi_h}$ and policy $\pi$,
\begin{equation}\label{eq:transition-operator}
	P_{\pi_h}(x,x') := \sum_a p(x|a,x)\pi_h(a|x), 
\end{equation}
to get  $d_h(x,a) = \left(\prod_{i=1}^h P_{\pi_{h}}d_0(x)\right)$.
Conversely, to match the state-action visitation by executing a policy, we can obtain a policy by marginalization:
\begin{equation}\label{eq:marginalization}
	\pi_h(a|x) = \frac{d_h(a,x)}{\sum_{a} d_h(a,x)} ~	 \text{and} ~ 	\bar{\pi}(a|x) = \frac{d(a,x)}{\sum_{a} d(a,x)},
\end{equation}
where the second case corresponds to the average case. In the latter case, the induced policy $\pi$ is stationary, while the former is non-stationary. We will drop the subscript $h$ from $d$ and $\pi$, and instead refer to them as $d_\pi$ and $\pi$, as the treatment for average and fixed horizon polytopes is essentially the same -- they differ only in the form of marginalization.
\vspace{-0.1cm}
\paragraph{Visitations and Trajectories} \looseness -1 Executing a policy $\pi$ for $H$ steps leads to a trajectory $\tau = \{x_0, a_0, \dots x_{H-1}, a_{H-1}\}$. Since both the policy and/or the environment can be stochastic, $\pi$ induces a distribution over trajectories of length $H$, which we denote by $\eta_\pi \in \mP$. We see that $\pi$ induces both $\eta_\pi$ and $d_\pi$, and in turn $d_\pi$ can be matched by a specific $\pi$. The distributions $d_\pi$ and $\eta_\pi$ can be related using a linear map $\bZ$ that we refer to as a \emph{conversion map} (see Fig.~\ref{fig:banner-pic}) as follows,
\begin{equation}\label{eq:operator}
\begin{aligned}
	d_\pi(a,x) & = \sum_{\tau \in \mT} \eta_\pi(\tau) \sum_{a',x' \in \tau} \delta_{a=a',x=x'} \\ 
			 & =  \sum_{a,x \in \mA \times \mX} \sum_{\tau \in \mT}  \#_{(a,x \in \tau)} \eta_\pi(\tau) = \bZ_{a,x}\eta_\pi,
\end{aligned}
\end{equation}
where $\#_{(a,x \in \tau)}$ refers to the number of times a state combination $(x,a)$ appears in $\tau$. The relation will be important as it allows us to directly optimize the distribution over trajectories (and hence policies) by optimizing state-action visitations. 
\vspace{-0.1cm}
\subsection{Loss Function}  \looseness -1
Using the \emph{conversion map} $\bZ$,  we can relate the objectives $F$ in terms of the empirical distribution of trajectories as $\eta$ with objective $U$ depending only on state-action visitations:
\begin{equation}\label{eq:objective-related}
	\min_{\eta_\pi \in \mP} F(\eta_\pi):= \min_{\eta_\pi \in \mP} U(\bZ \eta_\pi) = \min_{d_\pi \in \mD} U(d_\pi),
\end{equation}
where $d_\pi$ corresponds to the average state-action visitation of the policy $\pi$ (which needs to be neither Markovian nor stationary). Functions $U$ and $F$ are essentially the same; the only difference is in terms of the decision variables involved. 
These variables are related (in one way) via a linear map $\bZ$. This is possible due to the definition of $F(\eta_\pi)$ via the information matrix $I(\tau)$, which is additive in terms of state-action pairs in a trajectory $\tau$. For a formal derivation please see Lemma \ref{app:lemma:conversion} in Appendix \ref{app:basic}.
\vspace{-0.1cm}
\paragraph{Optimum} \looseness -1 We want that the empirical distribution over executed trajectories $\eta_T = \frac{1}{T} \sum_{t=1}^{T}\delta_{\tau_t}$ converges to the optimum of \eqref{eq:objective-related}, $\eta^* \in \arg\min_{\eta \in \mP} F(\eta)$, where $\delta_{\tau_t}$ is a delta-function supported on the trajectory $\tau_t$. The optimum  distribution over trajectories $\eta^*$ corresponds to a \emph{fixed} non-stationary (or stationary for the average case treatment) policy which comes from the set of optimal policies $\Pi^*$ - equivalent in terms of their value of $F$ (any one of them can be chosen).

By picking multiple different policies, each different in each episode, one might hope that one can perform better than a single optimal policy $\pi^*$, that induces $\eta^*$. However, in expectation, it cannot be improved upon due to convexity of $U$, as \vspace{-0.1cm}
\begin{align*} 
&\mE_{\tau_i\sim \pi_i}[F(\eta_T)]\stackrel{\eqref{eq:objective-related}}=\mE_{\tau_i\sim \pi_i}\left[ U \left( \bZ \eta_T \right)\right] \stackrel{\text{Jen.}} \geq U \left( \bZ  \mE_{\tau_i\sim \pi_i}\left[ \eta_T \right] \right) \\ 
&=  U\left(\frac{\bZ}{T}\sum_{t=1}^{T}\mE_{\tau_t\sim \pi_t}\left[ \delta_{\tau_t} \right] \right) \geq  U\left( \frac{\bZ}{T}\sum_{t=1}^{T} q_t \right) \stackrel{\text{opt.}} \geq U(\bZ \eta^*),
\end{align*}
where $q_t$ is the induced probability over trajectories due to the policy $\pi_t$. The last inequality follows as  $\frac{1}{T}\sum_{t=1}^{T}\bZ q_t$ is a convex combination of elements inside the state-action polytope, which is convex, and hence there exists a unique $\bZ \eta$ which can replicate it. Note that given $q_t$, we can find $d_{\pi_t} = \bZ q_t$ that in turns corresponds to the marginalized $\pi_t$. As $\eta^*$ is the optimum over $\mP$, $U(\bZ\eta)$ is larger than $U(\bZ \eta^*)$. The above calculation reveals that the value $F(\eta^*)$ is a good benchmark in expectation, and cannot be improved upon in expectation.

\vspace{-0.2cm}
\section{Convex RL: Non-adaptive Design}\looseness -1
Given a convex objective $U(d)$ (related to $F(\eta) = U(\bZ\eta)$) over polytope $d \in \mD$ (either the average or fixed horizon average polytope), we can solve for a policy $\pi^*$ which achieves the optimum of $d^*$ due to seminal works of \citet{Hazan2019} and \citet{Zahavy2021}. We refer to it as \emph{convex Reinforcement Learning (RL)}. It proceeds by solving a sequence of classical RL problems with a linear reward function that corresponds to the gradient of $U(d)$. We use it as a subroutine in our adaptive algorithm that we explain in Section \ref{sec:adaptive}. 
Using this method, we can find a fixed policy $\pi^*$ that matches the optimum value $U(d^*) = F(\eta^*)$. It does so by constructing a convex combination of base policies.
\vspace{-0.1cm}
\paragraph{Mixture Policy}\looseness -1 We refer to a convex combination of policies as  \emph{mixture policy}. It is a tuple consisting of a set $n$ base policies, and set of $n$ positive weights $\alpha_i$ that sum to one, $\pi_{\text{mix},n} = \{(\alpha_i,\pi_i)\}_{i=1}^n$. Such a mixture can be executed by first sampling an index $j \in [n]$ with probability equal to $\alpha_j$, and then evaluating policy $\pi_j$ for $H$ rounds. A property of mixture policies is that the state-action probabilities follow the convex combination of the policies such as $d_{\pi_{\text{mix},n}} = \sum_{i=1}^{n}\alpha_i d_{\pi_i}$. Note that like any other policy, a mixture policy can be summarized by a single policy via marginalization as in Eq.~\eqref{eq:marginalization}.
\vspace{-0.1cm}
\subsection{Convex RL as Frank-Wolfe}\label{sec:convex-rl}
\looseness -1 Convex RL can be solved via the Frank-Wolfe algorithm -- it incrementally constructs a mixture policy whose state-action distribution $d_{\pi_{\text{mix}},n}$ converges to $d^*$ as the number of mixture components $n$ increases. There are two distinct steps. The first is called \emph{density estimation}, corresponding to the evaluation of the gradient in Frank-Wolfe, and the second is \emph{policy search}, corresponding to the linear minimization oracle in Frank-Wolfe.
\vspace{-0.1cm}
\paragraph{Density Estimation}\looseness -1
Given a mixture policy, the density estimation oracle needs to estimate $d_{\pi_{\mix}}$. For discrete Markov chains (i.e., tabular MDPs), this amounts to a straightforward application of the operator $P_\pi(x',x)$ to $d_0$ as in \eqref{eq:transition-operator} for every component of the mixture policy. In particular for the fixed horizon setting $d_{\pi_i}(x) = \frac{1}{H}\sum_{h=1}^{H} \prod_{j=1}^h P_{(\pi_i)_j}d_0(x)$ where subscript $j$ denotes the iteration within the episode as $\pi$ is non-stationary for a fixed horizon. The overall mixture state-action visitation is the convex combination of all mixture components. Beyond discrete Markov chains, any state-visitation density can be estimated via sampling. Note that due to knowledge of transition operator $P$, this means simulation, not interaction with the environment.
\vspace{-0.1cm}
\paragraph{Policy Search} \looseness -1 Having estimated $d_{\pi_{\mix},n}$ with $n$ elements, we can now add an element $\pi_{n+1}$ such that the objective $U$ decreases. We linearize the objective $\nabla U(d)$ and solve the linear minimization oracle:
\[d_{\pi_{n+1}} = \argmin_{d_\pi \in \mD} \sum_{x,a} \nabla U(d_{\pi_{\mix}})(x,a) d(x,a) \]
This is a classical reinforcement learning problem, where $\nabla U(d_{\pi_{\mix}})$ plays the role of the reward function \citep{Puterman2014}. Hence, it can be solved by any RL solver such as value/policy iteration or linear programming. The newly found $d_{\pi_{n+1}}$ defines $\pi_{n+1}$ or vice versa depending on the RL solver used. The weight of new policy $\pi_{n+1}$, $\alpha_{n+1}$, is found via a line search as in Algorithm \ref{alg:1} or any other convergent step-size scheme for the Frank-Wolfe algorithm \citep{Jaggi2013}. The new mixture policy is then $\pi_{\text{mix},n+1} = \{( (1-\alpha_{n+1}) \alpha_i,\pi_i)\}_{i=1}^n \cup \{\alpha_{n+1}, \pi_{n+1}\}$. This algorithm is summarized in Alg.~\ref{alg:1} as \textbf{Convex RL}. \citet{Hazan2019} prove that in order to converge to $\epsilon$ optimality in terms of $U$, under the regularity conditions as in Assumption \ref{ass:regularity} (see Sec. \ref{sec:theory}), one needs $n \geq \mO( \frac{L}{\epsilon}\log(1/\epsilon))$ steps with step size $\alpha_n = \frac{\epsilon}{L}$ for all $n$. \footnote{A large smoothness parameter $L$ of $U$ is not a pressing issue as $n$ does not correspond to actual samples of the environment only the computational effort}

\vspace{-0.1cm}
\section{\textsc{Markov-Design}: Adaptive Design}\looseness -1 \label{sec:adaptive}
In the previous section, we discussed how to find a \emph{single policy} such that its state-action visitation probability minimizes a certain convex functional $U$. However, ultimately our objective depends on the empirical distribution $\eta_T = \frac{1}{T}\sum_{t=1}^{T}\delta_{\tau_t}$ of the executed trajectories. We now consider ways to generate trajectories: first the straightforward \textsc{non-adaptive} methods that execute a single policy multiple times, and then we describe the \textsc{adaptive} methods, which lie at the core of the contribution. 
\vspace{-0.1cm}
\subsection{Resampling from mixture $\pi^*$}\looseness -1
\vspace{-0.1cm}
\paragraph{Variant: \textsc{non-adaptive}}\looseness -1
As our final density is in the form of a mixture policy, we can either sample a component for each episode or summarize the policy by marginalization and execute it multiple times resulting into empirical $
eta_T$. The value $F(\eta_T) \rightarrow F(\eta^*)$ with probability $1-\delta$, depending on the gradient norm $B = ||\nabla F(\eta_T)||_\infty$, as $\mO(B\log(1/\delta)/\sqrt{T})$. In general, this strategy suffers from the {\em coupon collector problem}, where the same trajectories are resampled with non-zero probability. Namely this manifest itself in number of $T$ needed such that $B$ is well-behaved. For more details and formal statements, see Appendix \ref{app:resampling} and 
\ref{app:discussion}.

\vspace{-0.1cm}
\paragraph{Variant: \textsc{tracking}}\looseness -1 The \textsc{non-adaptive} variant is wasteful in that there exists a non-zero probability that the same base policy is executed multiple times, leading to similar trajectories (and hence redundant experiments). This can be avoided via \emph{tracking} -- closely following the empirical distribution of executed policies to the mixture found by convex RL. Namely, we choose the base policy $\pi_j$ such that $j =\argmax_i (\alpha_i - \hat{\alpha}_i)$, where $\hat{\alpha}_i$ corresponds to the empirical distribution of executed policies. This, however, can be wasteful, as it depends on how the mixture policy is decomposed. If the  mixture contains a lot of dissimilar policies, then this method can be very competitive. On the other hand, if the components are all very similar, then this is as wasteful as the \textsc{non-adaptive} variant. 
\vspace{-0.2cm}
\subsection{Adaptive Optimization on $\mP$} \looseness -1
A more elegant way to avoid resampling is to inform the choice of the next policy with the information about the executed trajectories from previous steps. To do this, we incrementally estimate the empirical distribution of the visited states from past trajectory distributions as $\bZ \eta_t$, where $\eta_t = \frac{1}{t} \sum_{i=1}^{t}\delta_{\tau_i}$. We then seek an addition to this empirical measure $d = \bZ q_t$ as $\frac{1}{1+t} d + \frac{t}{1+t} \bZ \eta_t$ which minimizes the objective $G_t(d) = U(\frac{1}{1+t} d + \frac{t}{1+t} \bZ \eta_t)$. Notice that the weighting $\frac{1}{1+t}$ is chosen such that the new empirical distribution over trajectories $\eta_{t+1} = \frac{1}{1+t} \delta_t + \frac{t}{1+t}\eta_t$ where $\delta_t \sim q_t$ remains still the average allocation over the executed trajectories. It is in a sense a greedy one-step change in the allocation that brings us closer to the optimal allocation.  
\vspace{-0.1cm}
\paragraph{Variants: \textsc{exact} and \textsc{one-Step}}\looseness -1 Depending on how we identify the distribution over trajectories $q_t$ (and the associated $d$), we distinguish two variants. The \textsc{exact} variant finds the exact element such that $q_t = \arg\min_{q\in \mP} U(\frac{1}{1+t} \bZ q_t + \frac{t}{1+t} \bZ \eta_t)$. The \textsc{one-step} variant simply runs the Convex MDP framework for one step, and plays the first component in the mixture policy $\pi_{\mix}$ as summarized in Algorithm \ref{alg:1}. The \textsc{one-step} can be seen as a poor man's version of \textsc{exact}, although there seems to be empirical benefits for the former as we will see. 
\vspace{-0.1cm} \paragraph{Uncertain Objectives} \looseness -1
\looseness -1 In some cases, the objective $F$ depends on an unknown quantity such as unknown variance $\sigma_{x,a}^2$. We can then optimize the worst case over the unknown compact parameter set $\Gamma$, as $F(\eta) = \sup_{\sigma_{a,x} \in \Gamma} F_{\sigma_{a,x}}(\eta)$. Since $F$ is convex and $\Gamma$ compact, the objective remains convex. Further details are given in Appendix \ref{app:robust}. 

\algdef{SE}[SUBALG]{Indent}{EndIndent}{}{\algorithmicend\ }%
\algtext*{Indent}
\algtext*{EndIndent}

\begin{algorithm*}
	\caption{\textsc{Markov-Design}}
	\label{alg:1}
	\begin{algorithmic}[1]
		\Require known Markov chain, $p(x'|x,a)$, Objective $F$ ($U$), Number of episodes $T$
		\While{$t \leq T$}
		\State \textbf{Convex RL:} solving $\min_{d \in \mD}G_t(d):= U( \bZ\eta_{t}\frac{t}{t+1} + \frac{1}{t+1} d)$ 
		\Indent
		\State $\pi_{\mix,1} = \pi_t; i = 1$, $d_{\pi_{\mix},1} = \bZ\eta_{t}$
		\Repeat \quad $i = i + 1$
		\State $\varpi_i =  \argmin_{\pi \in \Pi} \sum_{x,a} d_{\pi}(x,a)\nabla_{x,a} G_t(d_{\pi_{\mix},i}(x,a))$ \Comment{RL problem}
		\State  $d_{{\varpi}_i} = \textbf{Density Estimation}(\varpi_i)$ \Comment{keep track of visitations}
		\State $\alpha_i = \argmin_{\alpha \in \mR} G_t(\alpha d_{\pi_{\mix},i} + (1-\alpha)d_{{\varpi}_i})$ \Comment{line search}
		\State $\pi_{\mix,i+1} = (1-\alpha_i)\pi_{\mix,i}\cup\{(\alpha_i,\varpi_i)\}$ \Comment{update mixture policy}
		\State Update $d_{\pi_{\mix},i+1} = \alpha_id_{\pi_{\mix},i} + (1-\alpha_i)d_{{\varpi}_i}$ \Comment{update mixture policy visitations}
		\Until{convergence}
		\State \textbf{Choose} $\pi_t = \begin{cases}
			\textbf{Marginalize} \quad \pi_{\mix}	& \mbox{if} ~ \textsc{variant = exact} \\
			\varpi_1 \quad & \mbox{if} ~ \textsc{variant = one-step}
		\end{cases}$
		\EndIndent
		\State \textbf{Interaction}
		\Indent
		\State Sample trajectory from $\tau_t \sim \pi_t$ (also as $\delta_{\tau_t} \sim q_t$)
		\State $\bZ \eta_{t+1} = \bZ\frac{t}{t+1}\eta_{t} + \bZ\frac{1}{t+1}\delta_{\tau_t}$ \Comment{keep track of visited states}
		\EndIndent
		\EndWhile
	\end{algorithmic}
\end{algorithm*}

\vspace{-0.2cm}
\section{Convergence Theory}\looseness -1 \vspace{-0.1cm} \label{sec:theory}

\looseness -1 The \textsc{one-step} variant is closely related to the Frank-Wolfe algorithm on the space of trajectory distributions and the theory we develop for its convergence is largely based on it. The convergence cannot be linear unless step sizes are adjusted \citep{Lacoste2015}, however, our $\frac{1}{1+t}$ step-sizes are determined by the one-step update specific to this setting, and cannot be changed. 

We utilize the same convergence proof for the \textsc{exact} and \textsc{one-step} variant, although the two have different convergence behavior on the real problems. The regularity conditions under which we show convergence are summarized bellow. 
\begin{assumption}[Regularity]\label{ass:regularity} Let $F: \Delta_p \rightarrow \mR $ (and likewise $U$) be convex, differentiable, locally Lipschitz continuous in $\norm{\cdot}_\infty$, and locally smooth as, 
	\begin{equation}\label{eq:convexity-smoothness}\looseness -1
		F(\eta + \alpha h) \leq  			F(\eta) + \nabla F(\eta)^\top h + \frac{L_{\eta,\alpha}}{2} \norm{h}^2_2 .
	\end{equation}
	for $\alpha \in (0,1)$ and $\eta,h \in \Delta_p$, $L :=\max_{\eta,\alpha} L_{\eta,\alpha}$. 
\end{assumption}
\looseness -1 Note that the above differs from classical smoothness assumption, which we refer to as global smoothness. On its own is not sufficient to prove the desired rate of convergence we observe in practice. The problem is that the experiment design objectives can have global smoothness $L$ proportional to $\frac{T}{\lambda}$, which does not suffice to prove convergence. However, the usual behavior is that after a few initial steps (eg. ca. $\mO(p)$), the local smoothness constants drops to a small number. On top of that some objectives might not be smooth at all like E-design. To remedy both of these, one can apply classical smoothing technique due to \citet{Nesterov2005}, where the smoothed function $F_\mu$ is $\mu$ close to $F$, but with smoothness $L_{\mu} = L/(1+\mu L) \leq \frac{1}{\mu}$. Applying our algorithm on objective $F_\mu$ gives order $T$ optimal convergence rate with high probability. 
\begin{figure*}
	\includegraphics[width=\textwidth]{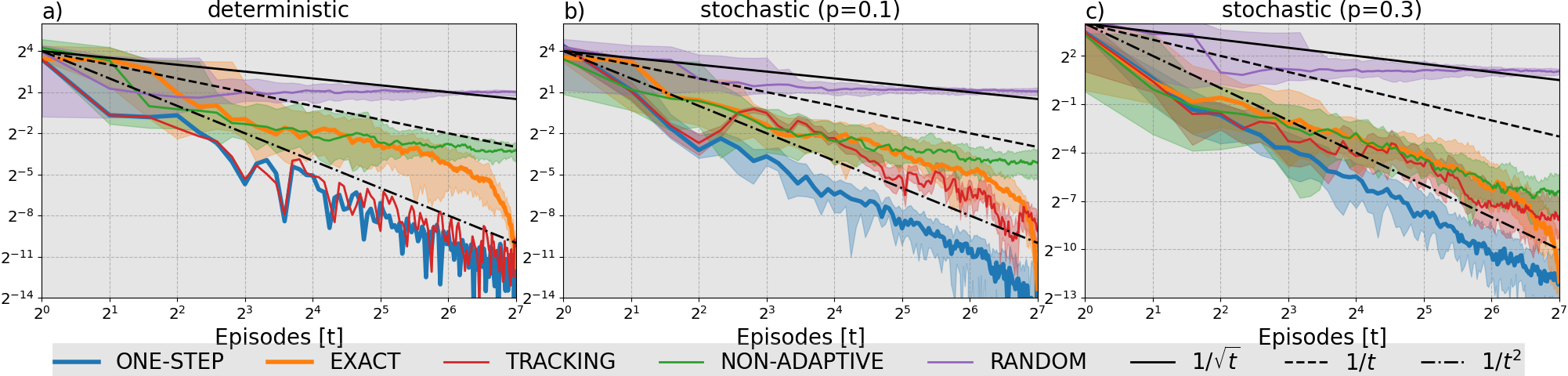}
	\caption{Gridworlds experiment: From a) to c) we increase the stochasticity of the environment. We plot the median with $10\%$ and $90\%$ quantiles over $20$ reruns of the method. Notice that the adaptive methods perform much better than the non-adaptive ones, and very quickly optimize the objective. For deterministic systems, the \textsc{one-step} method is deterministic as it coincides with the greedy method over a set of deterministic policies. Surprisingly, the  \textsc{one-step} method dominates the \textsc{exact} method for short horizons despite \textsc{exact} demonstrating linear convergence in later stages. The non-adaptive method exhibits slow convergence, while  random selection does not converge.}
	\label{fig:gridworld}
	\vspace{-0.8cm}
\end{figure*}

\begin{figure*}
	\centering
	\includegraphics[width=\textwidth]{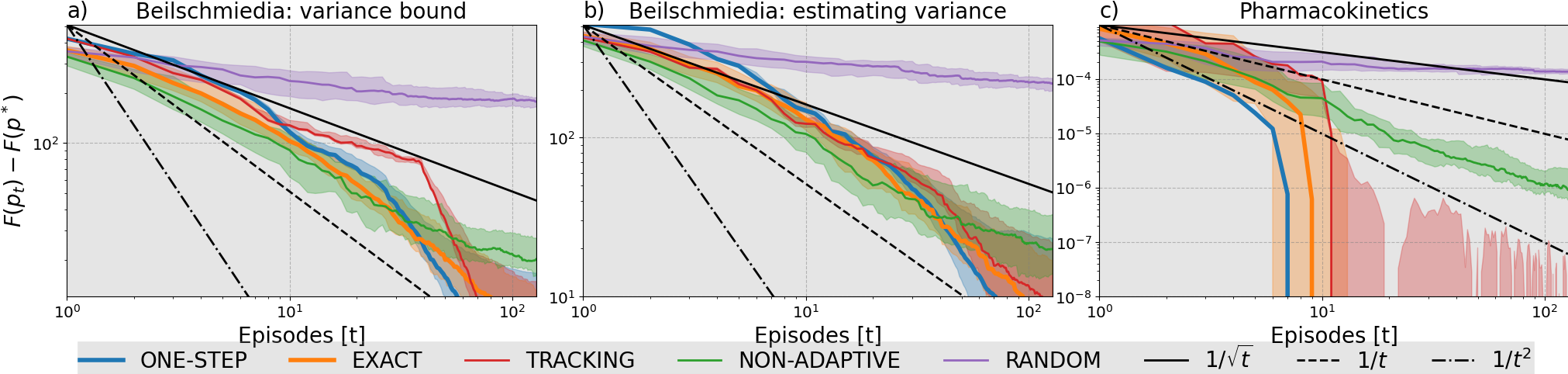}
	\caption{Beilschmiedia, and pharmacokinetics experiment: We report median of 20 reruns with $10\%$ and $90\%$ quantiles. In a) and b) we report the objective value as a function of episodes $t$ for $\sigma_{x,a}$ with a global upper bound, and one where we update the confidence set over $\sigma_{x,a}$ after each episode, respectively. Notice that we can converge to the objective faster in b) than a) due to better exploration policy (as it is more informed about values of variance). Note that \textsc{one-step} and \textsc{exact} are dominating the convergence. In c) we report the same objective for the pharmacokinetic example. \textsc{tracking} works very well for this problem, most likely due to a favorable mixture decomposition. }
	\label{fig:bels}
	\vspace{-0.9cm}
\end{figure*}

\begin{theorem} [Convergence \textsc{exact} and \textsc{one-step}]\label{thm:high-probability}
	Under Assumption \ref{ass:regularity}, with the smoothed objective using $\mu = \sqrt{\log T/T}$, the \textsc{exact} and \textsc{one-step} variants satisfy,
	\begin{equation*}
		F(\eta_T) - F(\eta^*) \leq  \mO\left(\frac{1}{T}\sqrt{\sum_{t=1}^{T} \norm{\nabla F(\eta_t)}_\infty^2 \log(T/\delta)} \right) 
	\end{equation*}
with $1-\delta$ probability over transition model and policy. 
\end{theorem}
Proofs are postponed to the Appendix \ref{app:cnvx}, and \citet{Nesterov2005} smoothing technique is reviewed in Appendix \ref{app:nesterov} for completness. With a reasonable upper bound on the gradient of $F$, the convergence is as $O(\frac{1}{\sqrt{T}})$ even for non-smooth objectives. 

However, the smoothing is not necessary as our experimental results point to. Particularly its use hampers the potentially fast convergence in expectation, which is the driving term in practice, especially for determinisitic Markov chains. Therefore, we state the next result in terms of the local Lipschitz constant $L_{\eta_t,1/t}$, which we conjecture, depend only logarithmically on $T$, leading to convergence $\mO(\frac{\log T}{T})$.
\begin{theorem}[Convergence \textsc{exact} and \textsc{one-step}]\label{thm:convergence-expectation}
	Under Assumption \ref{ass:regularity}, the \textsc{exact} and \textsc{one-step} variants satisfy
	\begin{equation*}
		\mE[F(\eta_t)] - F(\eta^*) \leq \mO\left(\frac{1}{t} \sum_{k=1}^T \frac{L_{\eta_k,1/k}}{1+k}\right)
	\end{equation*}
	for $t\leq T$ with expectation over transition model and policy. 
\end{theorem}
Showing non-trivial bounds on the smoothness and gradient remains a challenging problem not addressed even in the context of classical experiment design \citep{Zhao2022}. The only existing approaches that address this follow a particular initialization scheme which is infeasible for kernelized regime, and specialized for D-design \citep{Todd2016}. Nevertheless, we conjecture that for example for D-design, the gradient and smoothness are quickly of order $\operatorname{polylog}(T)$. Even with a proper initialization, one of the main challenges hampering the analysis of the algorithm is that the algorithm has fixed step-size $1/(1+t)$ and is not monotonically decreasing like other analysis of Frank-Wolfe \citep{Carderera2021}.
 
The consequence of the above theorems is that the suboptimality decreases as $F(\eta_T) \leq F(\eta^*) + T^{-1/2}$, which in turn means that the second moment of residuals in scalarization (see Eq. \ref{eq:cov}), $s(\bE_L)\leq \frac{F(\eta^*)}{T} + \mO(T^{-3/2})$, where the leading term in terms of $T$ depends on the optimal constant. Analog statements hold in expectation albeit with $\mO(T^{-2})$.

\vspace{-0.2cm}
\section{Applications \& Experiments }\looseness -1 \label{sec:experiments} \vspace{-0.1cm}
We present three applications of the proposed framework. We plot the suboptimality gap $F(\eta_t) - F(\eta_*)$. Further details of the experiments can be found in Appendix \ref{app:experiments}. Overall, the empirical convergence tends to have two phases. First slow, but then quickly, the rate of convergence is of the order $\mO(\frac{1}{T})$ to $\mO(\frac{1}{T^2})$. In very limited cases the \textsc{exact} variant exhibits linear convergence (see Figs.~\ref{fig:gridworld} and \ref{fig:bels}). Note that for statistical estimation the suboptimality at $t=T$ is decisive.
\vspace{-0.1cm}
\paragraph{Synthetic Gridworlds}\looseness -1 Consider a grid of a fixed height and width as in Fig.~\ref{fig:banner}, where possible actions are to move up, down, left, and right by one cell. We consider three levels of increasing stochasticity. We assume that with probability $1-p$, the action we take is executed as expected, and with probability $p$ an arbitrary valid action (up, down, left, right) is played instead. The unknown target $f$ corresponds to a linear function and the features $\Phi(x,a)$ correspond to a unit vector in a different direction for each state type (cf., pictogram). Thus, the states of the same type are completely correlated while different ones are not at all correlated. Fig.~\ref{fig:gridworld} shows that the convergence of our algorithm is of order $\mO(1/T^2)$ faster than our theory predicts, and $\mO(1/\sqrt{T})$ for the \textsc{non-adaptive} variant as we expect. With increasing stochasticity, the performance of the \textsc{tracking} variant reduces as it cannot adapt to executed trajectories from prior steps, while our adaptive methods can.

\vspace{-0.1cm}
\paragraph{Species monitoring: path planning} \looseness -1
Suppose we want to estimate the rate of occurrence of a particular species. To model this application, we use occurrence data of \emph{Beilschmiedia}, a tree genus native to Africa, from \citet{Baddeley2015}. We model the occurrence rate $f$ as a positive valued RKHS function that determines the rate of the spatio-temporal Poisson point process. The states in this work are sectors $X$ which are sensed. (See the map Fig.~\ref{fig:bels-map} in Appendix \ref{app:experiments}). We adopt the approach from \citet{Mutny2021a}, where $f$ is estimated with a penalized least squares estimator from count observations. The number of counts in region $X$ is distributed according to $y|X \sim \text{Poisson}(\int_{x \in X} f(x)dx)$. A peculiar property of the Poisson distribution is that its variance and mean are the same. Hence, the values of $\sigma_{a,X}^2 = \int_{x\in X} f(x) dx$, are unknown due to the unknown $f$. To model $f$, we use a squared exponential kernel that takes the slope $s_{x,y}$ and height $h_{x,y}$ of a point $(x,y)$ as inputs, as these are predictive of the habitat of Beilschmiedia. We assume that a drone can cover a certain trajectory of length $H$ before it has to return to the starting position. We use the D-design objective to maximize the information about $f$ everywhere in the domain. As $\sigma_{X,a}$ is unknown, we either run our algorithm with a known absolute upper bound on $\sigma_{X,a}$ (due to norm bound on $||f||_{\mH_k}$) as in Fig. \ref{fig:bels}a) or, alternatively, use confidence sets from \citet{Mutny2021a} on $f$ for adaptively collected data points. These allow us to construct confidence sets on $\sigma_{X,a}$ after each episode, and we take the upper bound of those. Results are shown in Fig.\ref{fig:bels}b). In both cases, the adaptive variants work well and their convergence rate is consistent with the theory we proved. 
\vspace{-0.1cm}
\paragraph{Group pharmacokinetics}\looseness -1
The goal of pharmacokinetics is to identify the rate of drug transport between the digestion system and other organ systems such as the circulatory system (blood). Such studies are designed for any new drug candidates to understand their absorption rates. The experiments are performed by drawing blood at specific time intervals, and inferring the medication concentrations over time. From this, the experimenters find the corresponding parameters $\gamma$ for the differential equation, which generate these concentration trajectories. We assume that the more accurately we can estimate the drug concentration trajectory over time for each patient, the more accurately we can estimate the parameters of the differential equation that generated these drug concentrations -- and hence focus our design on estimating the trajectories. More realistically, we also assume that the concentration in the blood has two components $f_i(t) = c_b(t) + g_i(t) + \epsilon$ where $\epsilon \sim \mN(0,\sigma^2)$ and $t$ is time. Hereby, $c_b(t)$ is the blood concentration of interest following a differential equation with parameters $\gamma$, and $g_i(t)$ is a patient-specific random variation that contaminates our measurements and we are not interested in inferring it per se. The differential equation is linear and hence forms  a linear constraint on the estimation as $L_\gamma (c_b) = 0$, where $L_\gamma$ is the linear differential operator. Note that this means that $c_b$ is in the null space of operator $L_\gamma$, which we denote $\bC_\gamma$ and the uncertainty of $c_b$ for a fixed $\gamma$ is then due to initial conditions only. While we set the initial concentration in our experiment by specifying dosage, we will model the initial conditions as unknown with little uncertainty, and then design an objective that reduces this uncertainty for a fixed $\gamma$. Since $\gamma$ is unknown, we will use a \emph{robust designs} of the experimental design objective, considering the supremum (resp.~infimum if negative) over reasonable values of $\gamma \in \Gamma$ as $F(\eta) = \sup_{\gamma \in \Gamma} \Tr\left[(\bC_\gamma \left( \sum_{\tau \in \mT} \eta(\tau) I(\tau) + (\lambda/T)  \bI \right)^{-1} \bC_\gamma^\top) \right]$ to define $F$. The policy, in this case, is a medical plan when to draw blood with a constraint that blood can be drawn only $5$ times per patient (episode) with reasonable separations (3 time steps). Each patient corresponds to an episode. We report the results in Fig.~\ref{fig:bels}c).
\vspace{-0.2cm}
\section{Conclusion}\looseness -1 \vspace{-0.2cm}
We introduced a novel algorithm \textsc{markov-design} for experiment design in Markov chains, capable of finding a set of exploratory policies that converge to an optimal allocation over trajectories to learn an unknown function of the states and actions in a known Markov chain. The algorithm solves a sequence of convex RL problems, which are informed by the previous trajectories of the agent. We proved the convergence rate of the method and its superiority to other approaches. We demonstrated its empirical performance in real-world problems and hope that this work will open a new avenue to study controlled Markov chains from an experimental design perspective. 
\newpage
\bibliography{refs_mdp.bib}
\bibliographystyle{apalike}
\newpage
\appendix
\onecolumn

\begin{center}
\Large{\bf{Supplementary Material: Active Exploration via Experiment Design \\in Markov Chains}} \hfil \\
\end{center}
\hfil \\
\hrule






\section{Additional Results}

\subsection{Relating state-visitations and trajectories}\label{app:basic}
%
We constructively prove that $F(\eta)$ and $U(d)$ can be related easily for scalarized information matrix objectives. 
\begin{lemma}[Conversion Mapping]\label{app:lemma:conversion}
	Let $F(\eta) = s(\bS(\eta))$, where $\bS(\eta) \in \mathbb{S}_+$  positive semi-definite cone of operators $\mH_k \rightarrow \mH_k$ s.t. $\bS(\eta) = \sum_{\tau \in \mT} \eta(\tau) I(\tau) + \lambda \bI$ and and $s:\mathbb{S}_+ \rightarrow \mR$, then there exists $U(d) = U(\bZ \eta) := F(\eta)$ which is equal to $U(d) = s(\sum_{a,x \in \mA\times \mX} d(a,x)\Phi(x,a)\Phi(x,a)^\top + \lambda \bI )$.
\end{lemma}
\allowdisplaybreaks
\begin{proof}
	The proof goes by construction, where we first construct $U$ and then verify it satisfies the desired property. The whole proof relies only on the additive property of information operator $I(\tau)$.
	
	\begin{eqnarray}
		U(d_\pi) & = & s\left(\sum_{a,x \in \mA\times \mX} d_\pi(a,x)\Phi(x,a)\Phi(x,a)^\top + \lambda \bI\right) \\
		&\stackrel{\eqref{eq:operator}} = & s\left(\sum_{a,x \in \mA\times \mX} \sum_{\tau \in \mT}\bZ(a,s,\tau)\eta_\pi(\tau)\Phi(x,a)\Phi(x,a)^\top + \lambda \bI\right)\\
		& = & s\left(\sum_{a,x \in \mA\times \mX} \sum_{\tau \in \mT}\#(a,x \in \tau)\eta_\pi(\tau)\Phi(x,a)\Phi(x,a)^\top + \lambda \bI\right)\\
		& = & s\left( \sum_{\tau \in \mT} \eta_\pi(\tau) \sum_{a,x \in \mA\times \mX} \#(a,x \in \tau)\Phi(x,a)\Phi(x,a)^\top + \lambda \bI\right)\\
		&\stackrel{\eqref{eq:cov}} = & s\left( \sum_{\tau \in \mT} \eta_\pi(\tau) I(\tau) + \lambda \bI\right) = F(\eta_\pi)
	\end{eqnarray}

\end{proof}

\subsection{Smoothing Technique for Convex Optimization}\label{app:nesterov}
In this section, we will briefly review, by now, the classical technique of convex optimization for non-smooth functions by \citet{Nesterov2005}. The development outlined here is inspired by \citet{Beck2012}. Suppose the function $F(\eta)$ is either smooth with a large constant $L$ or non-smooth. In what follows, the non-smooth case is recovered by letting $L \rightarrow \infty$. A central object component in defining the smoothing is the convex conjugate of $F$, $F^*$, 

\[ F^*(\zeta) := \max_{\eta \in \mP} \eta^\top \zeta - F(\eta) \]
In order to define a smoothed function $F_\eta$, we perform the reverse operation with added regularization
\begin{equation}\label{eq:smoothing-nesterov}
F_\mu(\eta) := \max_{\zeta} \zeta^\top\eta -F^*(\zeta) - \frac{\mu}{2}\norm{\zeta}^2,
\end{equation}
where the domain of $\zeta$ is everywhere where $F^*$ is finite. Equivalently it can be represented via dual reformulation as 
\begin{equation}\label{eq:smoothing}
F_\mu(\eta) = \inf_{\zeta \in \mP} F(\zeta)+ \frac{1}{2\mu}\norm{\zeta - \eta}_2,
\end{equation}
which is sometimes referred to as Moreau proximal smoothing \citep{Beck2012}. Due to the above definition it is clear that $F_\mu(\eta) \leq F(\eta) \leq F_\mu(\eta) + \mu \max_{x,y \in \mP}\norm{x-y}_2 \leq F_\mu(\eta) + 2\mu$. Hence by choosing sufficiently small $\mu$, we can ensure that minimizing $F_\mu$ will closely minimize $F$. In addition, $F_\mu$ is smooth, differentiable, and smooth as summarized in the following lemma. 

\begin{lemma}[Bounded smoothness]\label{lemma:bounded-smoothness} Let $F_\mu$ be smoothing of $F$ as in Eq. \eqref{eq:smoothing}. The function $F_\mu$ is $\frac{L}{1+\mu L}$-smooth. 
\end{lemma}
\begin{proof}
In order to show this we will use the definition in Eq. \eqref{eq:smoothing-nesterov} and the fact that convex conjugate has the property that a conjugate of $s$-strongly convex function is $1/s$-smooth and vice versa \citep{Borwein2005}.

Notice that $F^*$ is $1/L$ strongly-convex, then $F^* + \mu/2 \norm{\cdot}$ is $1/L + \mu$ strongly-convex. Conjugating these cause the function to be $\frac{1}{1/L + \mu} = \frac{L}{1+\mu L}$ smooth. As $L\rightarrow \infty$, the smoothness constant is $1/\mu$. 
\end{proof} 

As a corollary of the above definition, we also have a bound on the gradient that we will utilize later. 

\begin{lemma}[Bounded gradient\footnote{Special thanks to XY for the proof of this lemma.}]\label{lemma:bounded-gradient} Let $F_\mu$ be smoothing of $F$ as in Eq. \eqref{eq:smoothing}, then 
	\begin{equation}
		\norm{\nabla F_\mu(\eta)}_p \leq		\norm{\nabla F(\eta)}_p
	\end{equation}
for any $\eta \in \mP$, and $p \in [1,\infty]$. 
\end{lemma}
\begin{proof}
To prove this relation, we use representation in Eq. \eqref{eq:smoothing}. Note that $\nabla F_\mu(\eta) = \frac{1}{\mu} (\eta - \zeta^*)$, where $\zeta^*$ is where the infimum is realized. In particular, it holds that $\nabla F(\zeta^*) + \frac{1}{\mu}(\zeta^* - \eta ) = 0$, i.e.  $\zeta^* = (\bI + \mu \nabla F )^{-1}(\eta)$. The $\bI$ designates the identity operator. 

Now, 
\begin{eqnarray*}
	\norm{\nabla F_\mu (\eta)}_p & = & 	\frac{1}{\mu}\norm{\eta - \zeta^*}_p \\
	& = & \frac{1}{\mu}\norm{\eta - (\bI + \mu \nabla F)^{-1}\eta}_p \\
	& = & \frac{1}{\mu}\norm{(\bI + \mu \nabla F)^{-1}\left( (\bI + \mu \nabla F)\eta - \eta\right)}_p \\
	& \leq & \frac{1}{\mu}\norm{\left( (\bI + \mu \nabla F)\eta - \eta\right)}_p \\
	& = & \norm{\nabla F(\eta)}_p, 
\end{eqnarray*}
where in the fourth line we use that resolvent is non-expansive. 
\end{proof}

\subsection{A-design and D-design regularity}
The Assumption \ref{ass:regularity} is satisfied for the objectives in Table \ref{table:targets} with the exception of E-design which is non-smooth. Using the smoothing technique above one can make sure the conditions are satisfied for the optimization algorithm used in this work.

The proofs of the smoothness and convexity of A and V-design can be found in  Appendix, Lemma 8 of \citet{Borsos2020}. The convexity of D-design is clear from \citet{Pukelsheim2006}, the objective is also smooth in the regularized form, however, the constant has poor scaling. The following following lemma demonstrates it. 
\begin{lemma}\label{lemma:regularity}
	The objective $-\log\det\left( \sum_{x\in \mX} x x^\top \eta(x) + \lambda \bI \right)$ is $L$-smooth in terms of $\eta \in \Delta_{|\mX|}$, with $L = \frac{1}{\lambda^2}\lambda_{\max}((\bX\bX^\top)\circ(\bX\bX^\top))$, and additionally the $\lambda_{\max}(\bH(\eta)) \leq \norm{ \nabla F(\eta)}_2^2$, where $\bH$ is the Hessian.  
\end{lemma}
\begin{proof}
	The Hessian of the above objective is $\bH = (\bX \bV^{-1} \bX^\top) \circ (\bX \bV^{-1} \bX^\top) $, where $\bV = \bX^\top \bD(\eta) \bX + \lambda \bI$ and $\circ$ refers to the Hadamard product. This can be shown by taking the derivative, linearity of trace, and $\partial \bA^{-1} = \bA^{-1} (\partial \bA )\bA^{-1}$ for symmetric matrix. 
	
	The last part is by noting that all elements of $\bH_{ij}$ are positive, and strictly smaller than $\nabla_i F\nabla_jF$. The result follows by using Theorem 5.22 from \citep{Zhang2011}.
\end{proof}

\section{Formal Results and Proofs}
\subsection{Concentration Results}
First, we show that empirical measures constructed via sampling from a probability distribution can closely track its probability distributions using Azuma inequality. We will use this result in the convergence proofs later in Appendix \ref{app:cnvx}.

\begin{lemma}[Empirical measure concentration (linear)]\label{lemma:tracking2} Let $\{\eta_t\}_{t=1}$ be an adapted sequence of probability distributions on $\mX$, $\mP(\mX)$ with respect to filtration $\mF_{t-1}$. Likewise let $\{f_t\}_{t=1}$ be an adapted sequence of linear functionals $f_t:\mP(\mX) \rightarrow \mR$ s.t. $\norm{f_t}_\infty \leq B_t$. Also, let $x_t \sim \eta_t$, and $\delta_t(x) = \mathbf{1}_{x_t=x}$, then
	\begin{eqnarray}
		\pP\left( \Big|\sum_{t=1}^{T} f_t(\delta_t - \eta_t)\Big| \geq \sqrt{2\sum_{t=1}^{T}B_t^2\log\left(\frac{2}{\delta}\right)}
		\right) \leq \delta
	\end{eqnarray}
\end{lemma}
\begin{proof}
	Let $Q_t = f_t(\delta_{t} - \eta_t)$. In other words $Q_t = \int_{x \in \mX} f_t(x) (\delta_t(x) - \eta_t(x))dx$. Since,
	\begin{eqnarray}
		\mE[Q_t|\mF_{t-1}] & = & \mE[f_t(\delta_{t}(x) - \eta_t(x))|\mF_{t-1}] \stackrel{\text{lin.}} = f_t(\mE[\delta_{t}(x) - \eta_t(x)|\mF_{t-1}]) \\
		& = & f_t(\eta_t(x) - \eta_t(x)) = f_t(0) = 0.
	\end{eqnarray}
	and each $Q_t$ is bounded as $|Q_t|\leq \norm{f_t}_\infty \norm{\delta_t - \eta_t}_1\leq $
	$\norm{f_t} \leq B_t$, $Q_t$  it is martingale difference sequence. Consequently, we can use the generalized Azuma-Hoeffding \citep[Corr. 2.20]{Wainwright2019} inequality to control the sum,
	\[ \pP(|\sum_{t=1}^{T} Q_t| \geq \epsilon) \leq 2\exp\left(-\frac{2\epsilon^2}{4\sum_{t=1}^{T}B_t^2}\right), \]
	Setting $\delta/2 = \exp\left(-\frac{2\epsilon^2}{4\sum_{t=1}^TB_t^2}\right)$, gives $\epsilon^2 = 2\log(2/\delta)\sum_{t=1}^TB_t^2$.
\end{proof}

\begin{lemma}[Hoeffding on Hilbert space \citep{Lugosi2004}]\label{lemma:hilbert-hoeffding}
	Let $X_i \in L_2$ be zero mean s.t. $\norm{X_i}_2\leq s$, then
	for any $t \geq \sqrt{Ns^2}$
	\[ P( \norm{S_N} \geq t )\leq \exp(-(t-\sqrt{Ns^2})^2/(2(Ns^2)))\]
\end{lemma}
Pick $t = 2\sqrt{Ns^2\log(1/\delta)}$, then $	P( \norm{S_N} \geq \sqrt{2Ns^2\log(1/\delta)} )\leq \exp(-\log(1/\delta)) = \delta$

\subsection{Resampling convergence: Proofs}\label{app:resampling}
First, we prove the saturation of the resampling, \textsc{non-adaptive}, algorithm. Namely, given a policy $\pi^*$ which induces the optimal $d^*$ and hence optimal $\eta^*$, how quickly does the empirical $\eta_T = \frac{1}{T}\sum_{t=1}^{T}\delta_{\tau_t}$ converge to $\eta^*$? The proof trivially uses Lemma \ref{lemma:tracking2} and \ref{lemma:hilbert-hoeffding}, to show concentration. 

\begin{proposition}\label{prop:resampling}
	Under Assumption \ref{ass:regularity}, for resampling oracle with $T$ resamplings, and $\eta_T = \frac{1}{T}\sum_{t=1}^T \delta_{\tau_t}$,
		\[ F(\eta_T) - F(\eta^*) \leq  ||\nabla F(\eta_T)||_\infty\sqrt{ \frac{2\log(2/\delta)}{T}} \]
with probability $1-\delta$. For D-design, we can show that 
 		\[ \mE[F(\eta_T)] - F(\eta^*) \leq \frac{4||\nabla F(\eta_T)||_\infty^2}{T} \log(T) + 2 \]
where the expectation is over the sampling process. 
\end{proposition} 
\begin{proof}
	\[ F(\eta_T) - F(\eta^*) \stackrel{\text{convexity}}\leq \nabla F(\eta_T)^\top(\eta^*- \eta_{T}) = \frac{1}{T}\sum_{i=1}^T \nabla F(\eta_T)^\top(\eta^* - \delta_{\tau_t})\]
	Using $\norm{\nabla F(\eta_T)}_\infty = B$, and Lemma \ref{lemma:tracking2}, with $B_t = B/T$, we get
	\[F\left(\frac{1}{T}\sum_{t=1}^{T}\delta_t\right) - F(\eta^*) \leq \sqrt{4\sum_{t=1}^{T}\frac{B^2}{T^2}\log(2/\delta)}  \]

To show the bound in expectation,

 \begin{eqnarray*}
	F\left(\eta_T \right) - F(\eta^*) &\leq& \nabla F(\eta^*)^\top (\eta_T- \eta^*) + \frac{L_t}{2} \norm{ \left(\frac{1}{T}\sum_{t=1}^{T}\delta_t\right) - \eta^*}^2 \\
	&\stackrel{\text{Lemma}~\ref{lemma:regularity}}\leq& |\nabla F(\eta^*)^\top (\eta_T- \eta^*)| + \frac{\norm{\nabla F(z)}_2^2}{2} \norm{ \left(\frac{1}{T}\sum_{t=1}^{T}\delta_t\right) - \eta^*}^2 \\
	\mE[F\left(\eta_T \right)] - F(\eta^*) & \leq & \frac{B^2}{T^2}\mE\left[ \norm{\sum_{t=1}^{T}(\delta_t-\eta^*)}^2 \right] \\
	& \stackrel{\text{Lemma }\ref{lemma:hilbert-hoeffding}}	\leq & {B^2}{T^2} (2\sqrt{T\log(1/\delta)})^2 (1-\delta) + \delta 2T \\
	& \leq & \frac{4B^2}{T} \log(T) + 2
\end{eqnarray*}
where in the last step we used $\delta = 1/T$, and in the second to last we used that $\mE[\eta_T] = \eta^*$. 
\end{proof}

Note that the value depends on the constant $\norm{\nabla F(\eta_T)}_\infty$. We cannot globally control the value of the gradient, at least with a satisfactory constant, but for sufficiently large $T$, this value will be well-behaved. The following section \ref{app:discussion} provides intuition when this number becomes well-behaved. 

\subsection{Discussion on the regularity of the objective}\label{app:discussion}

The convergence rates in Theorems \ref{thm:high-probability} and \ref{thm:convergence-expectation} depend on the norm of the gradient $||\nabla F(\eta_t)||_\infty$ and local smoothness parameter $L_t$. Likewise, after $t$ resamplings the suboptimality of $\eta_t$ depends on it as in Proposition \ref{prop:resampling}. In this section, we would like to provide some intuition about these quantities.

To explore these quantities, we will use a completely degenerate example, where the Markov chain has $d$ states and $d$ action, where playing $i$th action leads to $i$th state, and the agent has to stay there as $H=1$. This corresponds to the classical experimental design problem where the Markov chain is trivial. On top of this, we assume that each $\Phi(x_i,a_i) = e_i$ where $e_i$ is unit vector in $\mR^{d}$, and we focus on D-design objective. 

In this case, the optimal policy $\eta^*$ puts equal mass on all trajectories, i.e. $\pP(\tau = \tau_i) = 1/d$ for all $i \in [d]$. 

\paragraph{Resampling} The gradient in this case is equal to: \[ ||\nabla F(\eta_t)||_\infty = \max_{i \in [d]} e_i^\top \left(\sum_{j=1}^{t}e_{\tau_j}e_{\tau_j}^\top + \frac{\lambda}{T} \bI\right)^{-1}e_i.\]
In this case, since the feature of each state-action is orthogonal to each other, the above is proportional linearly to $T$ only if $\{\tau_j\}_{j=1}^t$ does not span the whole basis. In order to do so, we need to sample each trajectory at least once. However, since the sampling is with replacement, we know due to classical \emph{coupon collector problem} we need to perform in expectation $d\log(d)$ resamplings in order to have each of them single. One can also express this using a high-probability event. Namely, let  $T_{\text{all}}$ be the time after all are sampled at least once. The probability $\pP(T_{\text{all}} \leq d\log(d) + d\log(1/\delta)) \geq 1-\delta$. Notice that $d\log(d)$ appears regardless of confidence score $\delta \in (0,1)$. 

\paragraph{Adaptive method} Notice in contrast that e.g. \textsc{one-step} or \textsc{exact} algorithm would always sample a different $e_{\tau_j}$ than previously collected in their next episode since they would always pick $\tau_j$ which maximizes the gradient (or any of them). Hence, already with $d$ steps, \emph{deterministically}, in this example, would lead to a bound on the gradient as $\frac{d}{1+\lambda/T} \leq d$. In subsequent steps, the value of the gradient can grow, however, it does not grow linearly with $T$. Notice that since the smoothness constant of the D-objective is related to the gradient squared due to Lemma \ref{lemma:regularity}, after $d$ steps even the local smoothness constant is bounded by $d^2$ in this example. 

This reveals that while the resampling and adaptive optimization have the same convergence in terms of $T$, in expectation or with high probability, their actual performance depends significantly on the auxiliarly constants such as gradient norm or local Lipschitz constant. 

Generalizing this procedure to more complicated feature spaces and Markov chain structure is a challenging problem. The above simplification suggests that adaptive methods perform some form of efficient initialization scheme. The closest in literature is the work of \citet{Todd2016} which provides an initialization scheme for D-design such that the gradient is bounded. They consider only objectives without regularization. Extending this technique to a kernelized setting is non-trivial and perhaps not as practically appealing as the methods suggested here, especially in the regularized scenario that we consider. Detailed analysis and bounding of these constants in more general settings is an open problem for future work. 


%

\subsection{Adaptive Method: Proofs} \label{app:cnvx}

In this section, we give proof of our adaptive methods. The proofs are related and essentially the same as those of \citet{Berthet2017} for Frank-Wolfe UCB albeit the source of error in gradients is different and handled differently. Due to the connection to Franke-Wolfe, we also know that the rate in expectation cannot be improved with this step size \citep{Lacoste2015}. We present \textsc{one-step} and \textsc{exact} variant in the following two theorems.

\begin{theorem}[Convergence high probability] \label{thm:high}
	Under Assumption \ref{ass:regularity}, using \citet{Nesterov2005} smoothing technique as in Appendix \ref{app:nesterov}  with $\mu = \sqrt{\log T/T}$, the suboptimality of \textsc{one-step} and \textsc{exact} variant can be bounded as, 
	\[	F(\eta_T) - F(\eta^*) \leq  \frac{1}{T}(F(\eta_1) - F(\eta^*)) + 2\sqrt{\frac{\log T}{T^{3}}} +\frac{1}{T} \sqrt{2\sum_{t=1}^{T} || \nabla F(\eta_t)||_\infty^2\log\left(\frac{2}{\delta}\right)} + 3\sqrt{\frac{\log T}{T}} + \frac{1}{T}\sum_{t=1}^T \Delta_t \]
	with probability $(1-\delta)$ over the stochasticity of policies and the Markov chain. 	 For \textsc{one-step} variant $\Delta_t = 0$ for all $t$, and for \textsc{exact} variant each subproblem is solved to near-optimality s.t. $\tilde{v}_t$ approximates
	\begin{equation*}
		v_t = \arg\min_{q\in \mP} F_\mu\left(\eta_t + \frac{1}{t+1} (q-\eta_t)\right), ~ \text{and} ~ F_\mu(\tilde{v}_t) - F_\mu(v_t)  \leq \Delta_t.
	\end{equation*}
\end{theorem}
\begin{proof}
	Let us refresh the definition of the following terms,
	
	\begin{equation}\label{eq:fw-update}
		q_t = \argmin_{q\in \mP} \nabla F_\mu\left(\eta_t\right)^\top q
	\end{equation}
	\begin{equation}
		\eta_\mu^* = \argmin_{q\in \mP} F_\mu(q), \quad and \quad 		\eta^* = \argmin_{q\in \mP} F(q).
	\end{equation}

	We will first show the proof of \textsc{one-step} variant and later show how \textsc{exact} can be reduced such that the same analysis applies. 
	\begin{eqnarray*}
		F_\mu(\eta_{t+1}) & = &		F_\mu\left(\eta_{t} + \frac{1}{t+1}(\delta_t-\eta_t)\right) \\
		&  \stackrel{L_\mu\text{-smooth}}  \leq & F_\mu(\eta_t) + \frac{1}{t+1}\nabla F_\mu(\eta_t)^\top(\delta_t - \eta_t) + \frac{L_\mu}{2(1+t)^2}\norm{\delta_t -  \eta_t }_2^2\\
		F_\mu(\eta_{t+1})& \stackrel{\text{bounded}} \leq & F_\mu(\eta_t) + \frac{1}{t+1}\nabla F_\mu(\eta_t)^\top(\delta_t - \eta_t) + \frac{L_\mu}{(1+t)^2} \\
		& = & F_\mu(\eta_t) + \frac{1}{t+1}\nabla F_\mu(\eta_t)^\top(q_t - \eta_t) + \frac{1}{t+1}\nabla F_\mu(\eta_t)^\top(-q_t + \delta_t) + \frac{L_\mu}{(1+t)^2} \\
		& \stackrel{\eqref{eq:fw-update}} \leq & F_\mu(\eta_t) + \frac{1}{t+1}\nabla F_\mu(\eta_t)^\top(\eta_\mu^* - \eta_t) + \frac{1}{t+1}\nabla F_\mu(\eta_t)^\top(-q_t + \delta_t) + \frac{L_\mu}{(1+t)^2} \\
		& \stackrel{\text{convexity}} \leq & F_\mu(\eta_t) - \frac{1}{t+1}(F_\mu(\eta_t)-F_\mu(\eta_\mu^*)) + \frac{1}{t+1}\nabla F_\mu(\eta_t)^\top(-q_t + \delta_t) + \frac{L_\mu}{(1+t)^2} \\
		F_\mu(\eta_{t+1}) - F_\mu(\eta_\mu^*)& \leq & F_\mu(\eta_t) - F_\mu(\eta_\mu^*) - \frac{1}{t+1}(F_\mu(\eta_t)-F_\mu(\eta_\mu^*)) + \frac{1}{t+1}\nabla F_\mu(\eta_t)^\top(-q_t + \delta_t) + \frac{L_\mu}{(1+t)^2} \\
		& \leq & \frac{t}{1+t} \left(F_\mu(\eta_t) - F_\mu(\eta_\mu^*)\right) +  \frac{1}{t+1}\nabla F_\mu(\eta_t)^\top(-q_t + \delta_t) + \frac{L_\mu}{(1+t)^2}
	\end{eqnarray*}
	where we used the shorthand $\epsilon_t = \nabla F_\mu(\eta_t)^\top(-q_t + \delta_t).$	Using shorthand $\rho_{t+1} = F_\mu(\eta_{t+1}) - F_\mu(\eta_\mu^*)$:
	\begin{eqnarray*}
		(F_\mu(\eta_{t+1}) - F_\mu(\eta_\mu^*))(t+1) & \leq &  t\left(F_\mu(\eta_t) - F_\mu(\eta_\mu^*)\right) +  \nabla F_\mu(\eta_t)^\top(-q_t + \delta_t) + \frac{L_\mu}{(1+t)} \\
		\rho_{t+1}(t+1) & \leq & \left(t\rho_t + \epsilon_t + \frac{L_\mu}{1+t}\right) \\
		\rho_{t+1}(t+1) - \rho_tt &\leq &\epsilon_t + \frac{L_\mu}{1+t}  \\
		\sum_{t=1}^{T-1} \rho_{t+1}(t+1) - \rho_t t & \leq &  \sum_{t=1}^{T-1}\epsilon_t +L_\mu\log T \\
		T\rho_T - \rho_1  & \leq & \sum_{t=1}^{T-1}\epsilon_t +L\log T \\
		\rho_T & \leq & \frac{1}{T}\rho_1 + \frac{1}{T}\sum_{t=1}^{T-1}\epsilon_t + \frac{L_\mu \log T}{T}
	\end{eqnarray*}
	Lastly as, we are interested in suboptimality of the actual function $\varrho_{t+1} = F(\eta_{t+1}) - F(\eta^*)$,  which can be bound as
	\begin{eqnarray*}
	\varrho_T & \leq & \rho_T + 2 \mu \leq 2\mu + \frac{1}{T}\rho_1 + \frac{1}{T}\sum_{t=1}^{T-1}\epsilon_t + \frac{L_\mu \log T}{T} \\
			& \stackrel{\text{Lemma}~\ref{lemma:bounded-smoothness}}\leq & \frac{1}{T}\rho_1 + \frac{1}{T}\sum_{t=1}^{T-1}\epsilon_t + \frac{\log T}{T}\frac{1}{\mu} + 2\mu \\
			& \stackrel{\text{Lemma}~\ref{lemma:tracking2}}\leq & \frac{1}{T}\rho_1 +\frac{1}{T} \sqrt{2\sum_{t=1}^{T}  ||\nabla F_\mu(\eta_t)||_\infty^2\log\left(\frac{2}{\delta}\right)} + \frac{\log T}{T}\frac{1}{\mu} + 2\mu \\
			& \stackrel{\mu = \sqrt{\log T/T}}\leq & \frac{1}{T}\rho_1 +\frac{1}{T} \sqrt{2\sum_{t=1}^{T} || \nabla F_\mu(\eta_t)||_\infty^2\log\left(\frac{2}{\delta}\right)} + 3\sqrt{\frac{\log T}{T}} \\
			& \stackrel{\text{Lemma}~\ref{lemma:bounded-gradient}}\leq & \frac{1}{T}\rho_1 +\frac{1}{T} \sqrt{2\sum_{t=1}^{T} || \nabla F(\eta_t)||_\infty^2\log\left(\frac{2}{\delta}\right)} + 3\sqrt{\frac{\log T}{T}} \\
			& \leq &	\frac{1}{T}\varrho_1 + 2\sqrt{\frac{\log T}{T^3}} +\frac{1}{T} \sqrt{2\sum_{t=1}^{T} || \nabla F(\eta_t)||_\infty^2\log\left(\frac{2}{\delta}\right)} + 3\sqrt{\frac{\log T}{T}}
	\end{eqnarray*}
	This proves the theorem for \textsc{one-step} variant. Now we will focus on the \textsc{exact} variant. The exact variant uses, 
	\begin{equation}\label{eq:update}
			v_t =  \arg\min_{q\in \mP}G_t(q) : = \arg\min_{q\in \mP} F_\mu\left(\eta_t + \frac{1}{t+1} (q-\eta_t)\right).
	\end{equation}	
	First, notice that due to the convexity of $F_\mu$, $G_t$ is convex as well. Also, note that $\frac{1}{1+t} \nabla F_\mu(\eta_t) = \nabla G_t(\eta_t)$. Using convexity, 
	
	\begin{equation}\label{eq:notice}
	 \frac{1}{1+t}\nabla F_\mu(\eta_t)^\top (v_t - \eta_t) = \nabla G_t(\eta)^\top (v_t - \eta_t) \leq G_t(v_t) - G_t(\eta_t) = G_t(v_t) - F_\mu(\eta_t) 
	\end{equation}

	Following the similar analysis notice, 
	\begin{eqnarray}
		F_\mu(\eta_{t+1}) & = & F_\mu(\eta_t + \frac{1}{t+1}(\delta_t - \eta_t))\\
		& \leq & F_\mu(\eta_t) + \frac{1}{1+t}\nabla F_\mu(\eta_t)^\top(\delta_t - \eta_t) + \frac{L_\mu}{(1+t)^2} \\
		& \leq & F_\mu(\eta_t) + \frac{1}{1+t}\nabla F_\mu(\eta_t)^\top(v_t - \eta_t) + \frac{1}{1+t}\nabla F_\mu(\eta_t)^\top(\delta_t - v_t) + \frac{L_\mu}{(1+t)^2} \\
		& \stackrel{\eqref{eq:notice}} \leq & F_\mu(\eta_t + \frac{1}{t+1}(v_t - \eta_t)) + \frac{1}{1+t}\nabla F_\mu(\eta_t)^\top(\delta_t - v_t) + \frac{L_\mu}{(1+t)^2} \\
		& \stackrel{\eqref{eq:update}} \leq & F_\mu(\eta_t + \frac{1}{t+1}(q_t - \eta_t)) + \frac{1}{1+t}\nabla F_\mu(\eta_t)^\top(\delta_t - v_t) + \frac{L_\mu}{(1+t)^2}  \label{eq:last-exact-line}
	\end{eqnarray}
	where the last line follows due to $v_t$ being a minimizer. The rest of the analysis is identical to the \textsc{one-step} variant. 

\end{proof}
The consequence of this theorem is that if the gradient $||\nabla F(\eta_t)||_\infty \approx \mO(1)$ and $\Delta_t \propto \mO(t^{-1/2})$, or zero, the overall complexity is of order $\mO\left( \sqrt{\frac{\log T}{T}}\right)$ limited by the concentration event. Now, in order not to be limited by the concentration event, we will look at the convergence in expectation. 

\begin{theorem}[Convergence expectation]\label{thm:expectation}
	Under Assumption \ref{ass:regularity}, the convergence of \textsc{one-step} and \textsc{exact} variant satisfies
	\[\mE[F(\eta_t)|\mF_{t-1}] - F(\eta^*) \leq  \frac{1}{t}(F(\eta_1)- F(\eta^*)) + \frac{1}{t}\sum_{k=1}^{t-1}\frac{L_{\eta_k,1/k}}{(1+k)} \]
	for $t \leq T$, where the expectation over stochasticity of policies $\{\pi_k\}_{k=1}^t$, and the Markov chain, and, where $\mF_{t-1}$ designates filtration up to time $t-1$.  For \textsc{one-step} variant $\Delta_k = 0$ for all $k$, and for \textsc{exact} variant each subproblem is solved to near-optimality s.t. $\tilde{v}_k$ approximates
	\begin{equation*}
		v_k = \arg\min_{q\in \mP} F\left(\eta_k + \frac{1}{k+1} (q-\eta_k)\right), ~ \text{and} ~ F(\tilde{v}_k) - F(v_k)  \leq \Delta_k.
	\end{equation*}
\end{theorem}

\begin{proof}
	Following the proof of Theorem \ref{thm:high}, we note that by taking expectation over $\mE[\epsilon_k] = \nabla F(\eta_k)^\top( -q_k + \mE[\delta_k]) = 0$. Notice that we apply the smoothness under Assumption \ref{ass:lipchitzness} not due to \citet{Nesterov2005} smoothing technique. 
	
	\begin{eqnarray*}
		F(\eta_{k+1}) & = &		F\left(\eta_{k} + \frac{1}{k+1}(\delta_k-\eta_k)\right) \\
		&  \stackrel{\eqref{eq:convexity-smoothness}}  \leq & F(\eta_k) + \frac{1}{k+1}\nabla F(\eta_k)^\top(\delta_k - \eta_k) + \frac{L_{\eta_k,1/k}}{2(1+k)^2}\norm{\delta_k -  \eta_k }_2^2\\
		& \leq & F(\eta_k) + \frac{1}{k+1}\nabla F(\eta_k)^\top(\delta_k - \eta_k) + \frac{1}{k+1}\nabla F(\eta_k)^\top(-q_k+ \delta_k)  + \frac{L_{\eta_k,1/k}}{(1+k)^2} \\
		\mE[F(\eta_{k+1})] & \leq &  F(\eta_t) + \frac{1}{k+1}\nabla F(\eta_k)^\top(q_k - \eta_k)+  \frac{L_{\eta_k,1/k}}{(1+k)^2} \\
		& \leq &  F(\eta_t) + \frac{1}{k+1}\nabla F(\eta_k)^\top(\eta^* - \eta_k)+  \frac{L_{\eta_k,1/k}}{(1+k)^2} \\
		&  \stackrel{\text{convexity}}   \leq &  F(\eta_t) - \frac{1}{k+1}(F(\eta_k) - F(\eta^*))+  \frac{L_{\eta_k,1/k}}{(1+k)^2} \\
		\mE[F(\eta_{k+1})] - F(\eta^*)& \leq &  F(\eta_t)- F(\eta^*) - \frac{1}{k+1}(F(\eta_k) - F(\eta^*)) +   \frac{L_{\eta_k,1/k}}{(1+k)^2} 
	\end{eqnarray*}
	Using shorthand $\rho_{k+1} = \mE[F(\eta_{k+1})] - F(\eta^*)$,
	\begin{eqnarray*}
		\rho_{k+1}(k+1)& \leq &   k\rho_k+ \frac{L_{\eta_k,1/k}}{(1+k)} \\
		\sum_{k=1}^{t-1} \rho_{k+1}(k+1) - k\rho_k & \leq & 	\sum_{k=1}^{t-1}  \frac{L_{\eta_k,1/k}}{(1+k)} \\
		\rho_t \leq \frac{1}{t}\rho_1 + \frac{1}{t}	\sum_{k=1}^{t-1}\frac{L_{\eta_k,1/k}}{(1+k)}
	\end{eqnarray*}
	
	The proof for the exact variant follows analogically as in the proof of Theorem \ref{thm:high}. 
\end{proof}
Notice that if $L_{\eta_k,1/k}$ could be globally bounded the suboptimality decreases as $\mO(\log T/T)$. 

\subsection{Robust and Uncertain objectives}\looseness -1 \label{app:robust}
The functional $\bC$ studied in the paper can itself depend on an unknown quantity, which we designate as $\gamma$ as $\bC_\gamma$, where $ \gamma \in \Gamma_t$, $t\in [T]$. For example, the value of $\sigma_{a,x}$ as in Eq.~\eqref{eq:cov} might not be known in advance. To deal with this complication, there are two approaches one can take:
\begin{itemize}
	\item \emph{Robust design:} take a supremum over the set $\Gamma$, and have a design that takes into account any possible values of $\gamma$. If $F$ is convex then so is the supremum over the compact set. 
	\item \emph{Sequential design:} amend the objective $F$ with the new value of the supremum if by executing a trajectory we can reduce the size of the set $\Gamma_t$. In the context of the example \eqref{eq:cov}, we can learn the variance from repeated samples and update over confidence set over them $\Gamma_t$.
\end{itemize}
We give examples for both of these design approaches in the experimental Section~\ref{sec:experiments}.

\subsection{Linear MDPs: Density Oracle}

\label{app:linear-mdp}
If the system is not tabular, the formula given in Sec.\ref{sec:convex-rl} to calculate $d$ cannot be used. However, one can always estimate $d$ via sampling which converges at a rate $N^{-1/2}$ with a number of trajectory samples $N$ given that we know the 'simulator' $P$. 

We provide a specific way to calculate the density for recently introduced class of Markov chains called \emph{linear MDPs} \citep{Jin2020}, which stipulate that $P(x'|x,a) = \mu(x')^\top \psi(x,a)$, where $\cdot^\top$ designates an inner product in an Euclidean space, and $\psi(x,a) \in \mR^{m}$, where $m$ is the dimension of the feature space.

Linear MDPs provide a way to improve the scalability when $\mX$ is larger or even infinite. In this case, the transition matrix $P_\pi $ with policy $\pi$ is equal to $ P_\pi(x',x) = \sum_{a} \mu(x')^\top \psi(x,a) \pi(a|x)  = \mu(x')^\top z_\pi(x)$, where $z_\pi(x) \in \mR^m$ is the mean embedding of $\psi(x,a)$ with probability distribution $\pi(a|x)$. Lets us define an operator $\bU: \mX \rightarrow \mR^{m}$ and $\bV_h:\mX \rightarrow \mR^{m}$, then $P_{(\pi_i)_h} = \bU \bV_{h,i}^\top$. Using that $d_0 = \delta_{x_0}$ and the definition from above $d_{\pi_i}(x) = \frac{1}{H}\sum_{h=1}^{H} \prod_{i=1}^h P_{(\pi_i)_j}d_0(x) =  \frac{1}{H}\sum_{h=1}^{H} \mu(x) \left( \prod_{i=2}^{h} \bV_{h,i}^\top \bU \right) z(x_0) = \mu \frac{1}{H}\sum_{h=1}^{H} \mu(x) \bM_{h,i} z(x_0)$, where $\bM\in \mR^{m\times m}$. Hence in order to estimate $d_\pi(x)$ we only need to multiply $m \times m$ matrices despite $|\mX|$ infinite.

\subsection{Relationship to Submodular optimization}\label{app:submodularity}

The objective $F$ can be reformulated as set function $H(S):2^{\mT} \rightarrow \mR$
defined on subsets $S = \{ \tau_1, \dots, \tau_t\}$ of the set of all possible trajectories. In addition, if $H$ is submodular, as is the case with $\log\det$  objective \citep{Krause2005}, we can benefit from the $(1-e^{-1})$ approximation guarantee for the greedy algorithm discovered by  \citet{Nemhauser1978} when solving cardinality constrained maximization $\max_{|S|\leq T} H(S)$. The objective $H$ upon reformulation can be related to the objective studied in this work as $H(S) = \log\det(\sum_{\tau \in S} I(\tau) + \lambda \bI) = d \log(|S|) - F(\frac{1}{|S|} \sum_{i=1}^{|S|} \delta_{\tau_i})$, where $F(\eta) = -\log\det\left( \sum_{\tau \in \mT}\eta(\tau)I(\tau)+ \lambda \bI/|S|  \right)$. The objective values are only are directly relatable only when $|S| = T$ due to different way of handling regularization.

Nevertheless, we can show that greedy selection from the ground formed by the power set of $\mT$, $2^\mT$ defines as: \[\tau_t = \argmax_{\tau \in \mT} H(S \cup \{\tau\})\] leads to the same update as if we optimized the convex-relaxation used within our framework, \[\argmin_{\tau \in \mT}F\left(  \frac{1}{|S|+1}\left(\sum_{i=1}^{|S|}\delta_{\tau_i} + \delta_{\tau}\right)\right),\] with the specially chosen $\frac{1}{1+t}$ step-size (as used in this work). The objectives cannot be easily compared, and hence greedy or convex optimization guarantees are different with the two viewpoints, but both frameworks lead to the \emph{same solutions}, i.e. same selection of trajectories, in the end when convex-RL is run with the varying regularization constant $\lambda_t = \lambda/t$; a minor technicality.

\paragraph{Stochastic set cover}
However, in our case, we cannot pick trajectories exactly (unless the system is deterministic). We pick distributions over $\mT$, $q \in \mP$. In particular a trajectory $\tau \sim q$ is generated from $q$ (or likewise associated $\pi$), and hence the objective we study is $\tilde{H}(\{\tau_1, \dots\} )= \mE_{\tau_i \sim \pi_i}[H( \{\tau_i, \dots, \tau_T\})]$ where the ground set changes to the \emph{set of all policies}. Submodularity is preserved upon nonnegative linear combinations (see \citet{Krause2014}), hence the objective $\tilde{H}$ as a function of policies is still submodular. In fact, this formulation exactly coincides with the stochastic set cover problem \citep{Golovin2011}. However the parallel to convex allocation used in this work is more complicated now. Notice that the greedy step of choosing a policy $\pi_{t+1}$ (equivalently $q_t$) to sample as an addition to already chosen trajectories $\{\tau_1, \dots, \tau_t \}$ due to $\{q_1 \dots, q_t\}$ is solving the following objective
\[ q_{t+1} = \argmax_{q \in \mP }= \mE_{\tau_{t+1}\sim q}[H(\{\tau_1, \dots\} \cup \{\tau_{t+1}\} )]. \]
If we were to look at the corresponding greedy update rule, where the convex measures are augmented by a single step (with the step size $1/(1+t)$) leading to the same update in terms of elements we would arrive at:
\begin{equation}\label{eq:greedy-relation}
	q_{t+1} = \arg\min_{q \in \mP} \mE_{\tau_{t+1} \sim q}\left[F\left(\frac{1}{t+1} (\sum_{i=1}^{t}\delta_{\tau_i} + \delta_{\tau_{t+1}} )\right)\right].
\end{equation}
Again the objective values are different but the sequence of elements chosen is the same. This objective is, however, different to the oracle we are using in Alg.~\ref{alg:1} step 2. Notice that our oracle looks for the best $q$ that minimizes \[F\left(\frac{t}{t+1}\eta_t +  \frac{1}{1+t}q\right) = F\left( \frac{1}{1+t}\left(\sum_{i=1}^{t} \delta_{\tau_i} + q\right) \right) = F\left( \frac{1}{1+t}\left(\sum_{i=1}^{t} \delta_{\tau_i} + \mE_{\tau_{t+1} \sim q}[\delta_{\tau_{t+1}}]\right) \right),\] which is related to the above greedy marginal gain via Jensen inequality. Unfortunately, the gap between these two can be large, and hence our method can be seen as a heuristic approximation to the greedy algorithm without an explicit guarantee on the greedy oracle. Only in the case when $q \in \mP$ where $\mP$ is restricted to trajectory probabilities due to the deterministic policies,  do the two methods coincide (as in the previous paragraph). This is because the expectation contains only one term and the Jensen gap is zero. This suggests that greedy formulation as in Eq. \eqref{eq:greedy-relation} is more powerful, however as we Sec.~\ref{sec:ee}, we still can be competitive with respect to the objective where the expectation is taken outside.

\subsection{Sequential Design and Uncertain Objectives}\label{app:unknown}
It turns out that we can be competitive to the \emph{true} value of $\gamma^* \in \Gamma_t$ in the set $\Gamma_t$ if the set decreases with time $t$. Its decrease influences the rate and the convergence can be guaranteed if it decreases at least as $1/\sqrt{t}$ if the size of $\Gamma_t$ is measured as Euclidean diameter and the set $\Gamma_t \subset \mR^m$ as we show in the following assumption. We do not directly show that $\Gamma_t$ decreases as this are different for each application and it might not be true in general. 
\begin{assumption}[Regularity with respect to the unknown]\label{ass:lipchitzness}
	Under Assumption \ref{ass:regularity} suppose further that either 
	
	\begin{enumerate}
		\item  $F_\gamma(\eta)$ is Lipschitz in $\gamma$ for \textsc{exact} variant:
		\begin{equation}\label{eq:lip}
			|F_{\gamma}(\eta) - F_{\gamma'}(\eta)| \leq U \norm{\gamma - \gamma'}_2
		\end{equation}
		where $U$ is independent of $\eta$.
		\item  $\nabla F_\gamma(\eta)$ is Lipschitz in $\gamma$ for \textsc{one-step} variant:
		\begin{equation}\label{eq:lip-gradient}
			\norm{\nabla F_{\gamma}(\eta) - \nabla F_{\gamma'}(\eta)}_2 \leq U \norm{\gamma - \gamma'}_2
		\end{equation}
		where $U$ is independent of $\eta$.
	\end{enumerate}
	
\end{assumption}
The proof of the following theorem is very much inspired by \citet{Berthet2017}, where this is essentially a Frank-Wolfe UCB algorithm. Note also that for the example of Poisson sensing from Sec.~\ref{sec:experiments} we know that the set $\norm{\gamma^* - \gamma_t}_2$ decreases at least as $1/\sqrt{t}$ since the objective is exactly designed to reduce the uncertainty of the $f$ and hence $\gamma$ and $f$ are the same this must be true. 
\begin{theorem}[Convergence for unknown] Under Assumptions \ref{ass:regularity} and \ref{ass:lipchitzness} suppose there exists $F_{\gamma^*}$, where $\gamma^* \in \Gamma_t$ for all $t\in [T]$ is unknown, then using smoothed target of $F_t(\eta) = \sup_{\gamma \in \Gamma_t} F_\gamma(\eta)$, denoted as $F_{t,\mu}$, with iteration varying smoothing parameter $\mu = \sqrt{\log T/T}$ in each round  $t \in [T]$ of \textsc{exact} algorithm we can show
	\begin{equation}
		F_{\gamma^*}(\eta_T) - F_{\gamma^*}(\eta^*) \leq \mO\left( (U+B)\sqrt{\frac{\log\left(\frac{2}{\delta}\right) + \log T }{T}}\right).
	\end{equation}
	Using oracle $ \min_{q \in \mP} \inf_{\gamma \in \Gamma_t} \nabla F_{\gamma,\mu}(\eta_t)^\top q$, where  $F_{\gamma,\mu}$ is the smoothed objective with $\mu = \sqrt{\log T/T}$ in \textsc{one-step} algorithm we can show, 
	\begin{equation}
		F_{\gamma^*}(\eta_T) - F_{\gamma^*}(\eta^*) \leq \mO\left( (U+B)\sqrt{\frac{\log\left(\frac{2}{\delta}\right) + \log T }{T}}\right)	
	\end{equation}
	as long as $\norm{\gamma^* - \bar{\gamma}_t} \leq \mO\left(1/{\sqrt{t}}\right)$, where $\bar{\gamma}_t$ achieves the supremum value in $\Gamma_t$. The value $B = \sqrt{\frac{1}{T}\sum_{t=1}^{T} \norm{\nabla F_{\gamma^*}(\eta_t)}_\infty^2}$ denotes the average gradient of the objective.  The statements hold with $1-\delta$ probability.
\end{theorem}

\begin{proof}\hfill \\
	\begin{description}
		\item[\textsc{exact}] In what follows we utilizing nearly identical proof technique as in Theorem~\ref{thm:high}. Namely, we continue where Eq. \eqref{eq:last-exact-line} left. Due to notation clutter, we drop the $\mu$ dependence on in $F$,
		\begin{eqnarray*}
			F_{\gamma^*}(\eta_{t+1}) & \leq & F_{\gamma^*}\left(\eta_t + \frac{1}{1+t}(v_t-\eta_t)\right) + \frac{1}{t+1}\nabla F_{\gamma^*}(\eta_t)^\top(-v_t + \delta_t) + \frac{1/\mu}{(1+t)^2} \\
			& \stackrel{\text{due to}~\sup_\gamma} \leq & F_{\bar{\gamma}_t}\left(\eta_t + \frac{1}{1+t}(v_t-\eta_t)\right) + U\norm{\bar{\gamma}_t - \gamma^*}_2 + \frac{1}{t+1}\nabla F_{\gamma^*}(\eta_t)^\top(-v_t + \delta_t)  + \frac{1/\mu}{(1+t)^2} \\
			& \stackrel{\eqref{eq:update}} \leq & F_{\bar{\gamma}_t}\left(\eta_t + \frac{1}{1+t}(q_t-\eta_t)\right) + U\norm{\bar{\gamma}_t - \gamma^*}_2 + \frac{1}{t+1}\nabla F_{\gamma^*}(\eta_t)^\top(-v_t + \delta_t) + \frac{1/\mu}{(1+t)^2} \\
			& \stackrel{\eqref{eq:convexity-smoothness}} \leq & F_{\bar{\gamma}_t}(\eta_t) + \frac{1}{1+t}\nabla F_{\bar{\gamma}_t}(\eta_t)^\top(q_t-\eta_t) + U\norm{\bar{\gamma}_t - \gamma^*}_2 + \frac{1}{t+1}\nabla F_{\gamma^*}(\eta_t)^\top(-v_t + \delta_t)   + \frac{2/\mu}{(1+t)^2} \\
			& \stackrel{\eqref{eq:fw-update}} \leq & F_{\bar{\gamma}_t}(\eta_t) + \frac{1}{1+t}\nabla F_{\bar{\gamma}_t}(\eta_t)^\top(\eta^*-\eta_t) + U\norm{\bar{\gamma}_t - \gamma^*}_2 + \frac{1}{t+1}\nabla F_{\gamma^*}(\eta_t)^\top(-v_t + \delta_t)   + \frac{2/\mu}{(1+t)^2} \\
			& \stackrel{\eqref{eq:convexity-smoothness},\eqref{eq:lip}} \leq & F_{\gamma^*}(\eta_t) - \frac{1}{1+t}( F_{\gamma^*}(\eta_t) -F_{\gamma^*}(\eta^*)) + 2U(1 + (1+t)^{-1})\norm{\bar{\gamma}_t - \gamma^*}_2 \\ & & + \frac{1}{t+1}\nabla F_{\gamma^*}(\eta_t)^\top(-v_t + \delta_t)   + \frac{2/\mu}{(1+t)^2} \\
			& \leq & F_{\gamma^*}(\eta_t) - \frac{1}{1+t}( F_{\gamma^*}(\eta_t) -F_{\gamma^*}(\eta^*)) + 4U\norm{\bar{\gamma}_t - \gamma^*}_2 \\ & & + \frac{1}{t+1}\nabla F_{\gamma^*}(\eta_t)^\top(-v_t + \delta_t)   + \frac{2/\mu}{(1+t)^2} 	
		\end{eqnarray*}
		The rest of the proof is similar to Thm. \ref{thm:high}, where $\rho_t = F_{\gamma^*}(\eta_t) - F_{\gamma^*}(\eta^*)$, $\epsilon_t =  \nabla F_{\gamma^*}(\eta_t)^\top(-v_t + \delta_t)$,
		\begin{eqnarray*}
			\rho_{t+1}(t+1) & \leq  & t\rho_t + 4U\norm{\bar{\gamma}_t - \gamma^*}_2 +  \epsilon_t + \frac{2/\mu}{(1+t)} 		
		\end{eqnarray*}
		Summing both sides on $t=1$ to $T-1$, 
		\begin{eqnarray*}
			T\rho_T - \rho_1 & \leq & \sum_{t=1}^{T-1}4U\norm{\bar{\gamma}_t- \gamma^*} + \epsilon_t + \sqrt{\frac{\log T}{T}} \\
			\rho_T  & \leq & \frac{1}{T} \rho_1 + \frac{4U}{\sqrt{T}} + B\sqrt{\frac{2\log\left(\frac{2}{\delta}\right)}{T}} + \sqrt{\frac{\log T}{T}} 
		\end{eqnarray*}
		where in the last step we used the fact that the distance $\norm{\bar{\gamma}_t- \gamma^*} \leq 1/\sqrt{t}$, and the Lemma \ref{lemma:tracking2} to bound the deviation of $\sum_{t=1}^{T}\epsilon_t$. Again, notice that due to to the smoothing the value is off by $\sqrt{\frac{\log T}{T}}$  which does correspond to the leading term. 
		
		\item[\textsc{one-step}]  Dropping $\mu$ from $F$ as above and utilizing nearly identical proof technique as in Thm.\ref{thm:high} we can show the following, 
		\begin{eqnarray*}
			F_{\gamma^*} (\eta_{t+1})& \stackrel{\eqref{eq:fw-update}} \leq & F_{\gamma^*} (\eta_t) + \frac{1}{t+1}\nabla F_{\gamma^*} (\eta_t)^\top(q_t - \eta_t) + \frac{1}{t+1}\nabla F_{\gamma^*} (\eta_t)^\top(-q_t + \delta_t)  \frac{1/\mu}{(1+t)^2} \\
			& = & F_{\gamma^*} (\eta_t) + \frac{1}{t+1}(\nabla F_{\gamma^*}(\eta_t) - \nabla F_{\bar{\gamma}_t}(\eta_t) + \nabla F_{\bar{\gamma}_t(\eta_t)} )^\top(q_t - \eta_t)  \\ & & + \frac{1}{t+1}\nabla F_{\gamma^*} (\eta_t)^\top(-q_t + \delta_t) + \frac{1/\mu}{(1+t)^2} \\
			& \stackrel{\eqref{eq:lip-gradient}} \leq & F_{\gamma^*} (\eta_t) + \frac{1}{t+1}(\nabla F_{\bar{\gamma}_t(\eta_t)} )^\top(q_t - \eta_t) + U\norm{\gamma^* - \bar{\gamma}_t}_2\norm{q_t-\eta_t}_2 \\ & & + \frac{1}{t+1}\nabla F_{\gamma^*} (\eta_t)^\top(-q_t + \delta_t) + \frac{1/\mu}{(1+t)^2} \\
			& \stackrel{\text{update rule}} \leq & F_{\gamma^*} (\eta_t) + \frac{1}{t+1}\nabla F_{\gamma^*} (\eta_t)^\top(\eta^* - \eta_t) + U\norm{\gamma^* - \bar{\gamma}_t}_2\norm{q_t-\eta_t}_2 \\ & & + \frac{1}{t+1}\nabla F_{\gamma^*} (\eta_t)^\top(-q_t + \delta_t) + \frac{1/\mu}{(1+t)^2} \\
		\end{eqnarray*}
		where in the last step, we used the update rule $ q_t = \argmin_{q \in \mP} \inf_{\gamma \in \Gamma_t} \nabla F_\gamma(\eta_t)^\top q$. The problem as appears now is reduced to the proof of Theorem \ref{thm:high}, as all terms feature only $\gamma^*$, and the only additional factor is $\sum_{t=1}^{t-1}U\norm{\gamma^* - \bar{\gamma}_t}_2\norm{q_t-\eta_t}_2 \leq 2U/\sqrt{t}$ upon summing as in Theorem \ref{thm:high}, which is $\mO(U/\sqrt{T} )$ due to the assumption in the theorem, which finished the proof. 
	\end{description}
	
\end{proof}

\subsection{Relationship to multi-agent systems}\label{app:blowup}
If we had $T$ agents that we were to release jointly at the same time, or equivalently we would like to plan jointly for $T$ episodes in one optimization step, we could increase the state-action space by assuming the new action space be $(\mX \times \mA) \times (\mX\times \mA) \dots $, where we do the product $T$ times. This increases the action-space exponentially in $T$, and we denote it $\tilde{\mX} \times \tilde{\mA}$. Visiting the state $x_i^{(t)},a_i^{(t)}$ in episode $t$ does not provide any different information and hence the observations for $\Phi(x_i^{(t)},a_i^{(t)}) = \Phi(x_i^{(t')},a_i^{(t')})$ for all $t,to \in [T]$. Hence, let us just drop the time superscript. 

Thus, if we were to consider the information matrix due to observing a states $\{(x_i^{(t)},a_i^{(t)})\}_{t=1}^T$ as a function of state-action visitation over $\tilde{d}$ over $\tilde{\mX} \times \tilde{\mA}$:
\[ U(\tilde{d}) = f \left(\sum_{x_i,a_i \in \mX \times \mA} \Phi(x_i,a_i)\Phi(x_i,a_i)^\top \left( \sum_{t=1}^{T} \tilde{d}(x_i^{(t)},a_i^{(t)})  \right) + \bI \lambda \right) \quad \tilde{d} \in \tilde{\mD}\]
where $f$ is the scalarization and $\tilde{\mD}$ is the average state-action polytope on $\tilde{\mX} \times \tilde{\mA}$. The equivalent solution is to pick $d(x_i^{(t)},a_i^{(t)}) = d(x_i,a_i)$ fixed for all $t$, which does not arise when optimizing jointly. We would arrive at an improved solution over the solution in Sec.~\ref{sec:convex-rl}. However, optimizing for multiple reruns jointly has two disadvantages a) blows up the state-action space combinatorially and b) does not adapt to the prior executed trajectories.
\section{Experiments: Further Information} \label{app:experiments}
In this section, we provide details of the experiments that we introduced in Sec.~\ref{sec:experiments}. Before we do so a couple of general comments. When optimizing over the polytope $\mD$, we always use an average-case polytope due to simpler implementations. This, from our experience, does not reduce the performance but provides a significant simplification in the code, and is a commonly made simplification found in other works such as \citet{Hazan2019}. Secondly, we never use \citet{Nesterov2005} smoothing technique, which is strictly not necessary since all our objectives are smooth only with globally large smoothness constants. Despite this, the algorithm performs very well suggesting the analysis is pessimistic in nature. 

To run the experiments we used a smaller server-class machine with $28$ CPU cores that we utilized for no more than 20 hours of active time. In general, this is a methodological paper and does not rely on any heavy calculation.

\paragraph{Synthetic grid} In this experiment we use $H = 20$. The exact location of different unit vectors as described in Sec.~\ref{sec:experiments} is visualized in Fig.\ref{fig:banner} with different pictograms. The optimization is run with \textsc{exact}-method such that when the duality gap is below $\epsilon = 0.05$ the optimization terminates. We used the exact line search with bisections to solve the line search problem. We always marginalize the policy before execution; except for \textsc{tracking} variant. The initial state is in the lower-left corner while the final is in the upper-right corner as in Fig.~\ref{fig:banner}.

\subsection{Beilschmeida}
First of all this dataset comes from the seminal work of \citet{Baddeley2015} and their R package. We closely follow the setup from \citet{Mutny2021a} to design the sensing problem. We assume a collection of Borel sets $\mA$, s.t. each set $A \in \mA$. They are all subsets of $[-1,1]^2$ in this case and are generated via hierarchical splitting of the domain. You can see all the sets in Fig.~\ref{fig:bels-map}, There are $16^2$ of these sensing actions. The Poisson process we model has an associated {\em intensity function} $\lambda(x)$ (plotted in Fig.~\ref{fig:bels-map2}),  where the number of events sensed during the unit time in $A$ is distributed as, 
\begin{equation}\label{eq:poison-counts}
	N(A) \sim \text{Poisson}\left( \int_{A} f(x) dx \right).
\end{equation}
We assume that $f \in \mH_k$, and that $f \geq 0$.  Given {\em number} of events in $A$, $n(A)$ for the duration period $\Delta$ in $t$ sessions we have  $\{(n(A_i), A_i)\} $ we use the heteroscedasdict least-squares estimator given the data which is defined via:
\begin{eqnarray}\label{eq:estimator-integral} 
		\hat{f}_t = \argmin_{f \in \mH_k, f \geq 0} \sum_{i=1}^{t} \frac{(\int_{A_i}f(x)dx-n(A_i))^2}{\sigma_i^2} + \frac{\gamma}{2} \norm{f}_{k}^2.
\end{eqnarray}
where $\sigma_i^2 = \int_{A_i}\bar{f}_{i-1}(x)dx$ is the upper confidence bound in the $i$th iteration. The bar denotes the element in the confidence set which maximizes the value of $\sigma_i^2$ over all plausible $f$. For the construction of the confidence set please refer to the \citet{Mutny2021a}.

Alternatively, we could use an absolute bound on $\sigma_i^2$ as, $\sigma_i \leq \int_{A_i}\bar{f}_{i-1}(x)dx \leq \operatorname{vol}(A_i) \max_x |f(x)| \leq \operatorname{vol}(A_i)\max_{x}k(x,x)\norm{f}_{{\mH}_k}\leq \operatorname{vol}(A_i) B\kappa$. The value of $B = 840$ for this experiment. See $\kappa$ below. Notice that this is a modeling parameter that depends on the volumes of the sets $A_i$ and these are scaled to be $[-1,1]$ which makes this value seem large, but if integrated over such small sets the rate will be of order $\mO(1)$. We use both approaches with known upper bound and with estimated $\sigma_i^2$ in Fig.~\ref{fig:bels}. 

To model the $\mH_k$, we use the squared exponential kernel that takes the slope $s_{x,y}$ and height $h_{x,y}$ of a point $(x,y)$ as inputs, as these are predictive of the habitat of Beilschmiedia, as $k((x,y),(x',y')) = \kappa \exp(-\frac{(s_{x,y}-s_{x',y;})^2 + (h_{x,y}-h_{x',y'})^2}{2\gamma}) $ where $\gamma = 0.1$, where $\kappa = 1000$ to match the scaling of the domain to $[-1,1]$. To relate this experiment to the grid-world example as in Figure \ref{fig:banner} we can cluster the sectors in the $(x,y)$ map due to the similarity, and close to `orthogonal` regions in space will belong to a different cluster. We report this clustering in Fig.~\ref{fig:bels-map} with 12 clusters. These clusters are not used in the algorithm; this is just for visualization purposes and to relate it to the toy experiment. The action space is the same as for the grid-world example with the difference that the initial and final states are in the lower-left corner. 

Notice also that due to RKHS assumption, $\int_{A_i} f(x) dx = \int_{A_i} f^\top \Phi(x) dx  = f^\top \varphi(A)$, where the new kernel due to evaluation function $\varphi(\cdot)$ is defined on Borel sets $A \in \mA$. We approximate the kernel using the triangle basis (with $400$ basis functions) as in \citep{Mutny2021a} to make sure that the positivity constraint can be implemented. We use their publicly available code for the fitting. 

\paragraph{Parameters of \textsc{markov-design}} The episode lengths is set to $H = 64$, $T = 128$, $|\mX| = 16^2$. We use D-design objective on the approximated $f(x) = \Phi(x)^\top f$ on the $400$-dim vector $f$, with $\lambda = 1$. The algorithm is run with line search and accuracy $0.05$ for the convex-RL part. We used the sequential variant of the algorithm as in Appendix \ref{app:unknown} for the case when variances are estimated after each episode.

\begin{figure}
	\centering
	\begin{subfigure}[t]{0.4\textwidth}
		\centering
		\includegraphics[width=\textwidth]{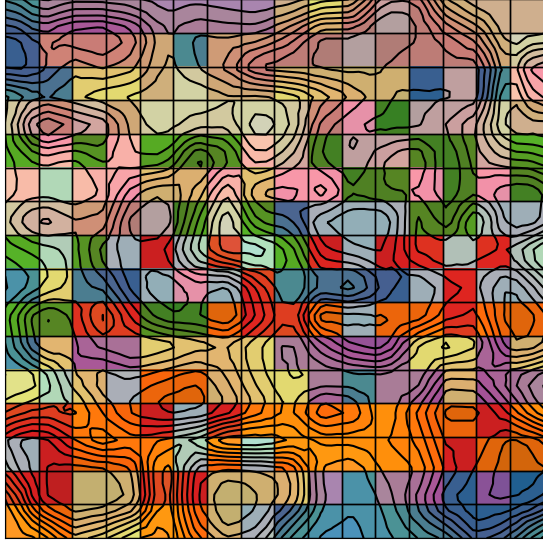}
		\caption{Beilschmiedia sectors via kernelized clustering in the color scheme. }
		\label{fig:bels-map}
	\end{subfigure}
	\begin{subfigure}[t]{0.42\textwidth}
		\centering
		\includegraphics[width=\textwidth]{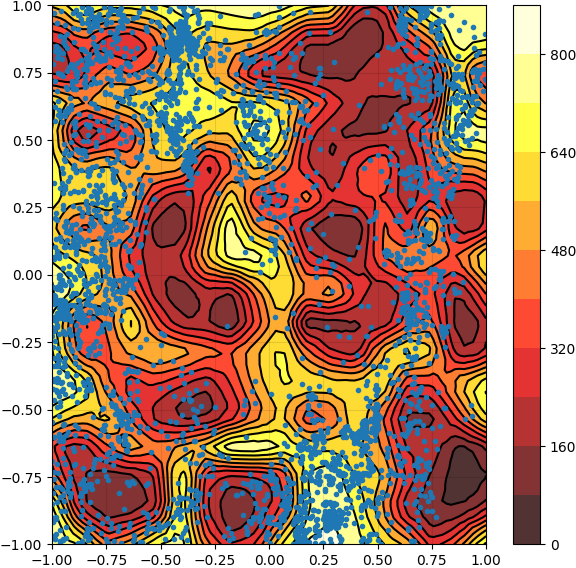}
		\caption{The rate function $f$ with observations in blue for the Beilschmiedia experiment.}
		\label{fig:bels-map2}
	\end{subfigure}
	\caption{ Beilschmiedia details:  In a) we report kernelized clustering due to the kernel $k$ of map sensed sectors. The same color corresponds to the same group overlayed with the rate function $f$. This map connects this example to the motivating example of Fig.~\ref{fig:banner}. b) Estimated rate function with the ground truth points.}
	\label{fig:bels-app}
\end{figure}
\begin{figure}\centering
	\begin{subfigure}[t]{0.45\textwidth}
		\centering
		\includegraphics[width=\textwidth]{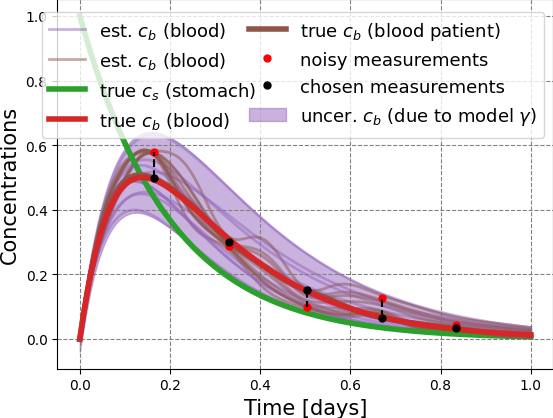}
		\caption{Medicament concentrations}
		\label{fig:pharma-example}
	\end{subfigure}
	\begin{subfigure}[t]{0.45\textwidth}
		\centering
		\includegraphics[width=\textwidth]{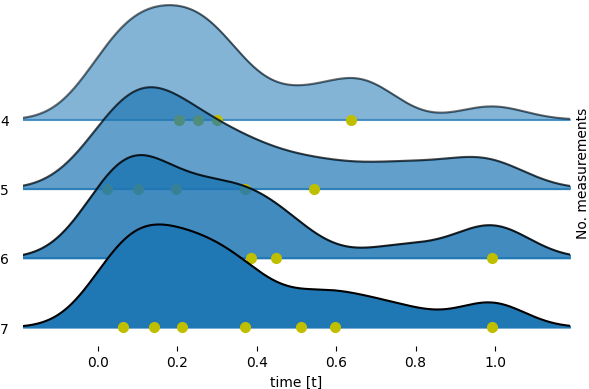}
		\caption{Distribution of measurements.}
		\label{fig:pharma-distribution}
	\end{subfigure}
	\caption{Pharmacokinetics: a) We plot the typical experimental measurements over time $t$. We see the two concentration values: $c_s(t)$ in stomach and $c_b(t)$ in blood. On top of that in brown, we see subject-specific values of $f_i$ which deviate from purple $c_b^{\gamma}$, where dependence on $\gamma$ is explicit, which correspond to possible $c_b$ values due to different $\gamma$ consistent with the data. The true blood concentration is in red. The uncertainty due to the unknown $\gamma$ is in shaded regions, which contains all plausible explanations to the data that we observed in red points. The chosen measurements are not due to our algorithm here. b): This plots a smoothed distribution of measurement locations over $T$ patients with different constraints on a total number of blood measurements. Notice that the measurements are denser in the initial stages of the experiment as a signal-to-noise ratio is much better. In yellow we depict one measurement plan for a patient (the first one). As one goes higher we vary the number of blood draws allowed in our Markov chain.}
	\label{fig:pharma-app}
\end{figure}
\subsection{Pharmacokinetics}
This dataset is synthetic, however, requires only 3 numbers which are choice reasonably with accordance to the prior works on this subject as in \citet{Mutny2022}. A similar problem occurs in \citet{Foster2021a}. One can study the details in the exhaustive book of \citet{Gabrielsson1995}. The explicit treatment of patient specific-contamination is novel to the  best of our knowledge. The grand goal of this experiment is to estimate the parameters of the linear differential equation which we do by first estimating the concentration levels of individual patients. We obtain data by sampling (drawing blood) at specific time intervals and estimating the concentration of medication in the blood. 

For each patient $i\in [T]$ we have the following oracle,
\[f_i(t) = c_b(t) + g_i(t) + \epsilon \quad \text{where} \quad  \epsilon \sim \mN(0,\sigma^2) \quad \text{and $t$ is time.}\]
where $g_i \sim \text{GP}(0,k)$, patient-specific random contamination, where $k$, is non-stationary variant of squared exponential kernel as $k(t,t') = \kappa \exp(-3t)\exp(-3t')k_{SE,\gamma}(t,t')$ with $\gamma = 0.1$, where $k_{SE}$ is the classical squared exponential kernel, and $\kappa = 0.05$. The small $\kappa$ models the fact that the variation is small and the non-stationary ensures that the variation decreases with increasing time. The value of $\sigma = 0.05$ is used as noise standard deviation. 

The $c_b$ denotes the concentration of a medication in the blood, which is shared among all patients. The $c_s$ denotes the concentration of the medication in the stomach. You can see in Figure \ref{fig:pharma-example} that the green curve decreases over time while $c_b$ increases. The dynamics is governed by a \emph{linear differential equation}:
\begin{equation}
	\frac{d}{dt}		\begin{pmatrix}
		c_s(t) \\ c_b (t)
	\end{pmatrix} - \begin{pmatrix}
		-a & 0 \\ b & - c
	\end{pmatrix} \begin{pmatrix}
		c_s(t) \\ c_b(t)
	\end{pmatrix} = 0 
\end{equation}
We stack the coefficients $\gamma = (a,b,c)$, and denote the above operator as $L_\gamma$ acting on $(c_s(t), c_b(t))$. We do not know the values of $\gamma$ as we want to estimate them, however we know that they varying in a box constraint $\Gamma$ which for this experiment is $\Gamma = [4,6]\times[9,11]\times[9,11]$. 

Now, note that as $L_\gamma (c_s(t), c_b(t)) = 0$, we see that the real solutions lie in the kernel of $L_\gamma$. If we discretize the operator over a time-horizon $t \in [t_0,t_1]$ and enforce this constraint and apply it to $c_b(t) = \Phi(t)^\top u_b$ (likewise $c_s(t) = \Phi(t)^\top u_s$), where $\Phi(t)$ is the evaluation functional of a kernel $k$, in this example corresponding to sufficiently rich kernel to reproduce any trajectory (in this case squared exponential as it is an universal kernel). We know that $c_b$ is such that it has to be in the kernel of the operator of discretized operator $\bL_\gamma$. We can stack the orthogonal rows of the kernel of discretized operator $\bL_\gamma$ to get $\bC_\gamma$. Thus $(u_b,u_s) = \bC_\gamma^\top v$, where $v \in \mH_k\setminus \operatorname{span}(\bL_\gamma)$. Thus the only unknown of the trajectory is the $\bC_\gamma$ which corresponds to unknown initial conditions given a fixed $\gamma$. 

We model the initial conditions of blood to be known $c_b(0)=0$, and to be slightly unknown for stomach in order to have a well-defined problem $(c_s(0) - c_{\text{dose}})^2 \leq \epsilon$. For a fixed $\gamma$ we can design an experiment that reduces the uncertainty in the initial condition the most. Despite the initial condition not being the core \emph{unknown} problem here we use them as a tool to get an informative data collection scheme. So given $\gamma$, the objective would be A-design as minimization of $\Tr\left[(\bC_\gamma \left( \sum_{\tau \in \mT} \eta(\tau) I(\tau) + (1/T)  \bV_0 \right)^{-1} \bC_\gamma^\top) \right]$, 
where $\bV_0 = \bI + [\Phi(0),0][\Phi(0),0]^\top$, where the second term corresponds to prior on initial conditions of $c_s$. As $\gamma$ is unknown we take $\inf_\gamma$ over the above objective to get the overall objective which should yield a data collection scheme which is good for any $\gamma \in \Gamma$. The scheme we just explained were complete if there was not any contamination due to $g_i$. As this one is present, we can absorb it into noise as correlated noise, 
\[ f_i(t) = (\Phi(t))^\top (u_b) + \epsilon + \Phi_N(t)^\top g_i \]
where $\Phi_N(t)$ is the evaluating functional of random element $g_i$. Stacking all the problem vectors together we have $f_i = (u_s,u_b)$ and features $\tilde{\Phi}(t) = (\Phi(t), \Phi(t))$. In order to implement $\Phi(t)$ for convenience of implementation we use high fidelity approximation squared-exponential kernel from \citet{Mutny2018b} known as \emph{Quadrature Fourier Features} (QFF) with $\gamma = 0.05$ with $m = 150$. The $\Phi(x)_N$ correspond to the non-stationary kernel we explained for patient-specific contamination. 

We absorb the noise $g_i$ into the information matrix by scaling the features properly. First, consider all times $t$ (i.e. all states), corresponding to matrix $\bSigma(t,t') = \Phi_N(t)^\top\Phi_N(t')$, then we define, new features $\Psi(t) = \sum_{t'}\bSigma^{-1/2}(t,t')\tilde{\Phi}_{t'}$, which are in turn define the information matrix $I(\tau) = \sum_{t \in \tau} \frac{1}{\sigma^2} {\Psi}(t){\Psi}(t)^\top$ used in for this experiment. This makes sure that features are properly scaled to reflect the heteroscedastic noise due to $g_i$. A more detailed treatment of heteroscedastic noise is provided in \citet{Kirschner2018}, which inspired the definition of the information matrix in this work.

We show the smoothed distribution of measurement location in Fig.~\ref{fig:pharma-distribution} and we see that the locations are distributed in the initial phases of the experiment as we would expect from the pharmacological perspective \citep{Gabrielsson1995}. In this time frame the signal-to-noise ratio is higher than later when the concentration drops below a noise level. Having the trajectories estimated we can then run maximum likelihood estimation to identify the parameter gamma from the inferred trajectories given $\gamma$ similarly as done in \citet{Mutny2022} or by \citet{Foster2021a}. 

\paragraph{Parameters of \textsc{markov-design}} The episode lengths are set to $H = 128$, $T = 128$, $|\mX| = 128\times 5\times 3$, where we allow $5$ draws of blood in the experiment over the whole time duration, and the spacing between blood draws needs to be at least $3$; in other words, we have to wait for 3 time-steps before we draw blood again. We use the A-design objective on the values approximated with linear functional $\bC_\gamma$ as $\gamma$ is not known we take the robust version as explained in App.\ref{app:robust}, we use $\lambda = 0.5$. The algorithm is run with line search and accuracy $10^{-8}$ due to different scaling of the objective ($\bC_\gamma$ is small as it contains normalized rows). The domain $t \in [0,1]$ for defining $\bC_\gamma$ is discretized with $128$ points.

\end{document}